\documentclass{article}

\usepackage[dvipsnames]{xcolor}
\usepackage{microtype}
\usepackage{graphicx}
\usepackage{subcaption}
\usepackage{booktabs} 
\usepackage{hyperref}
\usepackage{url}
\usepackage{multirow}

\usepackage[preprint]{icml2026}

\usepackage{amsmath}
\usepackage{amssymb}
\usepackage{mathtools}
\usepackage{amsthm}
\usepackage{thm-restate}

\usepackage[capitalize,noabbrev]{cleveref}
\usepackage{enumitem}
\usepackage[capitalize,noabbrev]{cleveref}

\theoremstyle{plain}
\newtheorem{theorem}{Theorem}[section]

\theoremstyle{definition}
\newtheorem{definition}[theorem]{Definition}
\newtheorem{assumption}[theorem]{Assumption}
\theoremstyle{remark}
\newtheorem{remark}[theorem]{Remark}
\DeclareMathOperator*{\esssup}{ess\,sup}
\DeclareMathOperator*{\essinf}{ess\,inf}

\newcommand{\diff}{\mathop{}\!\mathrm{d}}

\usepackage[textsize=tiny]{todonotes}

\newcommand\independent{\protect\mathpalette{\protect\independenT}{\perp}}
\def\independenT#1#2{\mathrel{\rlap{$#1#2$}\mkern2mu{#1#2}}}

\usepackage{tikz}
\DeclareRobustCommand\circled[1]{\tikz[baseline=(char.base)]{\node[shape=circle,draw=darkblue,minimum size=0.4cm,inner sep=0pt,fill=lightblue] (char) {\fontfamily{phv}\selectfont \scriptsize \textbf{#1}};}}
\DeclareRobustCommand\circledpurple[1]{\tikz[baseline=(char.base)]{\node[shape=circle,draw=violet,minimum size=0.4cm,inner sep=0pt,fill=violet!20] (char) {\fontfamily{phv}\selectfont \scriptsize \textbf{#1}};}}
\DeclareRobustCommand\circledgreen[1]{\tikz[baseline=(char.base)]{\node[shape=circle,draw=ForestGreen,minimum size=0.4cm,inner sep=0pt,fill=ForestGreen!20] (char) {\fontfamily{phv}\selectfont \scriptsize \textbf{#1}};}}

\definecolor{lightblue}{HTML}{DAE8FC}
\definecolor{darkblue}{HTML}{003366}

\icmltitlerunning{Partial identification for causal effects under network interference}

\begin{document}

\twocolumn[
  \icmltitle{Causal Inference on Networks under Misspecified Exposure Mappings: \\A Partial Identification Framework}

  \begin{icmlauthorlist}
    \icmlauthor{Maresa Schröder}{lmu,mcml}
    \icmlauthor{Miruna Oprescu}{cor}
    \icmlauthor{Stefan Feuerriegel}{lmu,mcml}
    \icmlauthor{Nathan Kallus}{cor,net}
  \end{icmlauthorlist}

  \icmlaffiliation{lmu}{LMU Munich}
  \icmlaffiliation{mcml}{Munich Center for Machine Learning}
  \icmlaffiliation{cor}{Cornell University}
  \icmlaffiliation{net}{Netflix}

  \icmlcorrespondingauthor{Maresa Schröder}{maresa.schroeder@lmu.de}

  \vskip 0.3in
]

\printAffiliationsAndNotice{}

\begin{abstract}
Estimating treatment effects in networks is challenging, as each potential outcome depends on the treatments of all other nodes in the network. To overcome this difficulty, existing methods typically impose an \emph{exposure mapping} that compresses the treatment assignments in the network into a low-dimensional summary. However, if this mapping is misspecified, standard estimators for direct and spillover effects can be severely biased. We propose a novel \emph{partial identification framework for causal inference on networks} to assess the robustness of treatment effects under misspecifications of the exposure mapping. Specifically, we derive sharp upper and lower bounds on direct and spillover effects under such misspecifications. As such, our framework presents a novel application of causal sensitivity analysis to exposure mappings. We instantiate our framework for three canonical exposure settings widely used in practice: (i) weighted means of the neighborhood treatments, (ii) threshold-based exposure mappings, and (iii) truncated neighborhood interference in the presence of higher-order spillovers. Furthermore, we develop \emph{orthogonal estimators} for these bounds and prove that the resulting bound estimates are valid, sharp, and efficient. Our experiments show the bounds remain informative and provide reliable conclusions under misspecification of exposure mappings.
\end{abstract}

\vspace{-0.5cm}
\section{Introduction}

\begin{figure*}[t]
  \centering
  \includegraphics[width=1\textwidth]{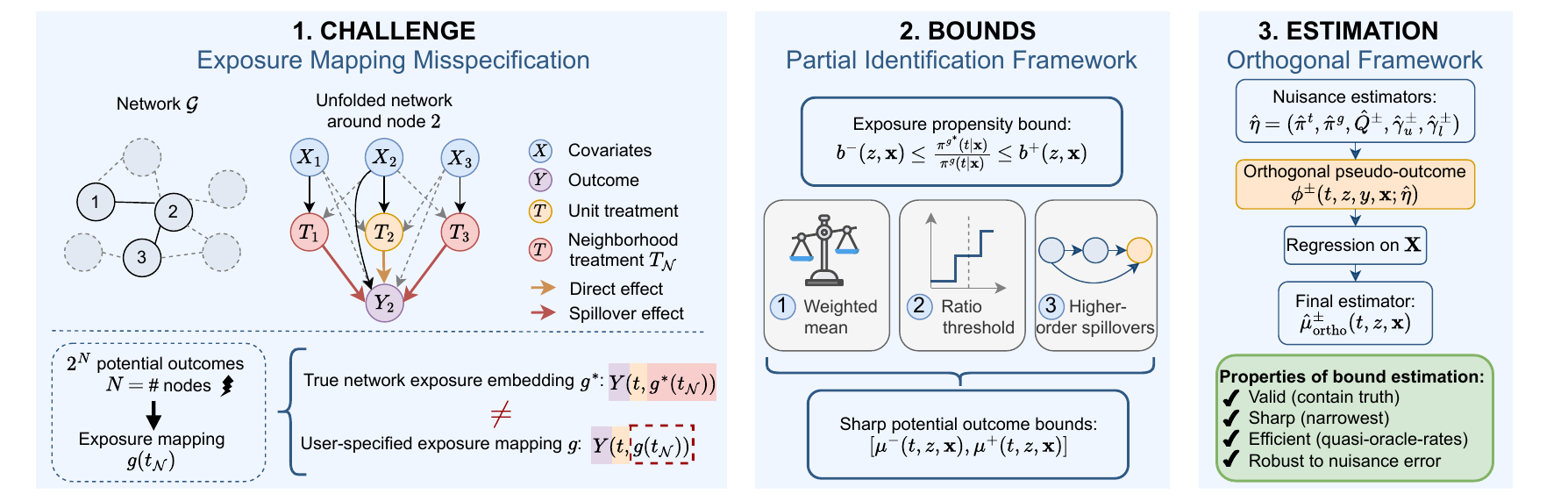}
  \vspace{-0.6cm}
  \caption{\textbf{Overview and contribution.} (1)~\textbf{Challenge}: Each unit’s outcome depends on the entire treatment assignments in the network. Existing methods compress the treatment assignments into a low-dimensional summary via an exposure mapping $g$. However, $g$ might differ from the true mapping, thus leading to biased effect estimates. (2)~\textbf{Bounds}: We model misspecification through an exposure-propensity bound and derive treatment effect bounds for 3 common exposure mappings. (3)~\textbf{Estimation}: We estimate the bounds via an orthogonal two-stage framework. Our estimated bounds are valid, sharp, efficient, and robust to nuisance estimation errors.}
  \vspace{-0.35cm}
  \label{fig:motivation}
\end{figure*}

Estimating treatment effects in network settings is crucial for evaluating policy effectiveness and designing personalized interventions \cite{Viviano.2025}. However, classic methods from causal inference assume \emph{no} interference between units, meaning that the outcome of each unit is independent of the treatments from other units, but this assumption is often violated in social networks \cite{Forastiere.2021, Matthay.2022, Ogburn.2024}. 

\textbf{Example:} \emph{Consider a public health intervention that targets individuals aged 60 and above to encourage COVID-19 vaccination \cite{Freedman.2026}. Individuals targeted by the interventions may be more likely to get vaccinated themselves ({direct effect}), while also influencing decisions of their friends through social interactions ({spillover effect}).} 

Causal inference in network settings is fundamentally challenging due to \emph{interference} \cite{Anselin.1988, Forastiere.2021}. In such settings, outcomes depend not only on a unit's own treatment but also on the treatments of connected units, which leads to spillover effects. More importantly, this directly violates the stable unit treatment variable assumption (SUTVA), so that standard treatment effects are no longer identified. Instead, causal inference on networks must consider \emph{all treatments of the complete network.} 

A na{\"i}ve approach to handle interference is to condition on \emph{all} treatment assignments in the network. However, this grows exponentially in the number of units $N$ and is thus highly impractical. A common workaround is to simplify the problem and impose an \emph{exposure mapping} that compresses the networks’ treatment assignment into a low-dimensional summary (e.g., number or share of treated neighbors in the network) \cite{Aronow.2017}. This approach is widely used in the literature \citep[e.g.,][]{Forastiere.2021,Forastiere.2022,Ogburn.2024} (see Sec.~\ref{sec:related_work} for a detailed overview). \textit{However}, the exposure mapping must be specified \emph{a priori} based on domain knowledge of how spillovers propagate through the network. In many applications, this mechanism is only partially understood, so the exposure mapping is likely to be misspecified, resulting in biased effect estimates. 

As a remedy, we develop a novel \textbf{partial identification framework for inference on networks} to assess the robustness of treatment effects under misspecifications of the exposure mapping through sensitivity analysis (see Fig.~\ref{fig:motivation}). Specifically, we derive sharp upper and lower bounds on the (conditional) potential outcomes and treatment effects when the exposure mapping is misspecified. We instantiate our framework for three canonical exposures used in practice: (i) weighted means of neighborhood treatments, (ii) threshold-based exposure mappings, and (iii) truncated neighborhood interference in the presence of higher-order spillovers. We make the following contributions:\footnote{Code available at GitHub:
\url{https://github.com/m-schroder/ExposureMisspecification}} 
\vspace{-0.2cm}
\begin{description}
\vspace{-0.2cm}
\item[\circledpurple{1}] Our bounds fulfill several desirable theoretical properties. In particular, our bounds are \emph{sharp} and \emph{valid}, in that they provide the narrowest possible intervals that contain the true outcome given a specific level of exposure mapping misspecification. 
\vspace{-0.25cm}
\item[\circledpurple{2}] We provide a model-agnostic \emph{orthogonal bound estimator} that achieves quasi-oracle  convergence rates and remains \emph{robust} to nuisance model misspecification. 
\vspace{-0.25cm}
\item[\circledpurple{3}] We provide guarantees for the \emph{estimated} bounds in that they remain \emph{sharp}, \emph{valid}, \emph{efficient}. 
\end{description}
\vspace{-0.3cm}
Our experiments show the bounds are informative and foster reliable decision-making under exposure mapping misspecification.


\vspace{-0.2cm}
\section{Related work\protect\footnote{We provide an extended overview of related work in Supp.~\ref{sec:appendix_related_work}.}}
\label{sec:related_work}
\vspace{-0.1cm}

\textbf{Interference:} Literature allowing for interference between different units broadly considers two distinct scenarios: (i)~\emph{partial or cluster-based interference}, and (ii)~\emph{network interference}. We focus on the latter, namely, network interference. In contrast, partial interference assumes that interference happens within groups or clusters but not across different groups \citep[e.g.,][]{BargagliStoffi.2025, TchetgenTchetgen.2012, Qu.2024, Fang.2025, VanderWeele.2014}. However, this group-level interference is unlikely to hold in real-world settings, making the assumption restrictive and often invalid. 

Network interference is less explored, where the majority of methods only apply to randomized controlled trials (RCTs), instead of observational data \citep[e.g.,][]{Alzubaidi.2024, Aronow.2017, Leung.2020, Savje.2021}. Methods targeted to network interference generally assume correct knowledge of a network treatment-summarizing function $g$, called \emph{exposure mapping} \citep[e.g.,][]{Chen.2024b, Forastiere.2021, Forastiere.2022, Liu.2023, Ogburn.2024, Sengupta.2025}. There are three exposure mappings that are commonly applied in the literature: (i)~the (weighted) neighborhood mean exposure, (ii)~a thresholding function, and  (iii)~one-step neighborhood exposure. We provide a formal definition in Section~\ref{sec:setup}. \emph{However, these methods fail to provide correct estimates of treatment effect if the exposure mapping is misspecified.}

\textbf{Misspecification in network inference:}
Only a few works consider causal effect estimation under misspecification. Of those, one stream focuses on causal effect estimation when the network structure is unknown, i.e., when there is uncertainty about the existence of certain edges in the network \citep[e.g.,][]{Egami.2021, Bhattacharya.2020, Savje.2024, Weinstein.2023, Weinstein.2025, Zhang.2023, Zhang.2025}. In contrast, our work assumes the \emph{network is fully known}, but the \emph{exposure mapping is misspecified}.

We are aware of two works that allow for a potential misspecification of the exposure mapping, but in a simplified setting (see Supplement~\ref{sec:appendix_related_work} for details). \citet{Leung.2022} considers approximate neighborhood interference, by allowing treatments assigned to units further from the unit of interest to have potentially nonzero, but decreasing, effects on the unit's outcome. \citet{Belloni.2022} consider causal effect estimation under an unknown neighborhood radius.  Unlike our work, both methods are restricted to the specific type of misspecification and are only applicable to \textit{average} causal effects. \emph{In contrast, our proposed framework incorporates misspecification not only of the neighborhood radius, but also covers other types of misspecification in exposure mappings as well as a broad set of causal estimands.} 

\textbf{Research gap:} In sum, there is \underline{no} general framework for bounding potential outcomes and treatment effects under various types of exposure mapping misspecifications for experimental and observational data. This is our contribution.

\vspace{-0.1cm}
\section{Setup}
\label{sec:setup}
\vspace{-0.1cm}

\textbf{Notation:}
We use capital letters $X$ to denote random variables, with realizations $x$ (lowercase letters). The probability distribution of $X$  is represented by $\mathbb{P}_X$, though we omit the subscript when the context makes it clear. For discrete variables the probability mass function is written as $P(x)=P(X=x)$ and for continuous variables, the probability density function as $p(x)$. In our work, we build upon the potential outcomes framework~\cite{Rubin.2005}. We provide an overview of the notation in Supplement~\ref{sec:appendix_notation}.

\vspace{-0.1cm}
\subsection{Network setting}

We follow the standard setting for causal inference on networks \cite{Chen.2024, Forastiere.2021}. We consider an (undirected) network of known structure given by the sets of nodes $\mathcal{N}_{\mathcal{G}}$ with $|\mathcal{G}|=N$ and edges $\mathcal{E}$ with $(i,j)=(j,i)$ for $i,j \in \mathcal{G}$. For each node $i$, we define a partition of the network as $(i,\mathcal{N}_i, \mathcal{N}_{-i})$, where $\mathcal{N}_i$ defines the \emph{neighborhood} of node $i$, i.e., all nodes $j$ connected to $i$ by an edge $(i,j) \in \mathcal{E}$ and $\mathcal{N}_{-i}$ the complement of $\mathcal{N}_i$ in $\mathcal{G}$. We refer to $|\mathcal{N}_i|=n_i$ as the \emph{degree} of node $i$. We omit the subscript whenever it is obvious from the context.

Every unit $i$ consists of the following variables: a treatment $T_i \in \{0,1\}$, confounders $X_i \in \mathcal{X}^d$, and an outcome $Y_i \in \mathcal{Y}$. We allow the treatment assignment to depend on (i)~the unit's own covariates $\mathbf{X}_i=X_i$  [\emph{homogeneous peer influence}], or (ii)~both unit $i$'s and it's neighbors covariates $\mathbf{X}_i=(X_i, X_{\mathcal{N}_i})$ [\emph{heterogeneous peer influence}], where we additionally assume that every node has the same degree $n$. The treatment assignment of unit $i$ is independent of the other units' treatment assignments given the covariates $\mathbf{X}$. We denote the unit \emph{propensity score} $P(t \mid \mathbf{X=x})$ as $\pi^t(\mathbf{x})$.


\vspace{-0.1cm}
\subsection{Exposure mappings}

As standard in treatment effect estimation on networks \citep[e.g.,][]{Chen.2024b, Forastiere.2021}, we assume the existence of an \emph{exposure mapping} $g: [0,1]^{n_i} \rightarrow \mathcal{Z}$ with $z_i := g(t_{\mathcal{N}_i})$, which a summary function of the treatments assigned to the neighbors of node $i$. Here, $z_i$ is assumed to be a sufficient representation to capture how neighbors' treatments affect the outcome $Y_i$. Therefore, the potential outcome is fully represented by $Y_i(t_i,z_i)$ and depends on the binary treatment $t_i$ as well as the discrete or continuous neighborhood treatment $z_i$. We denote the network propensity by $\pi^g(z\mid \mathbf{x}):= p(g(t_{\mathcal{N}_i})=z\mid \mathbf{X_i}=\mathbf{x_i})$.

Prior literature commonly builds upon three different types of \emph{exposure mappings} (see Section~\ref{sec:related_work}): 

\circled{1} \textbf{Weighted mean of neighborhood treatments:} A large body of existing works assume the summary function $g$ to represent the mean of the neighbors' treatments \citep[e.g.,][]{Belloni.2022, Chen.2024b, Chen.2025, Forastiere.2021, Forastiere.2022, Jiang.2022, Leung.2020, Ma.2021}. However, in many cases, such as social network or spatial interference settings, it is reasonable to assume that the neighbors' effects vary, e.g., with closeness in friendship or spatial distance \citep[e.g.,][]{Giffin.2023}. Hence, this motivates to formalize the underlying exposure mapping $g^{\ast}$ as a \emph{weighted} mean of treatments.
\vspace{-0.4cm}
\begin{figure}[h]
    \centering
    \includegraphics[width=1\linewidth]{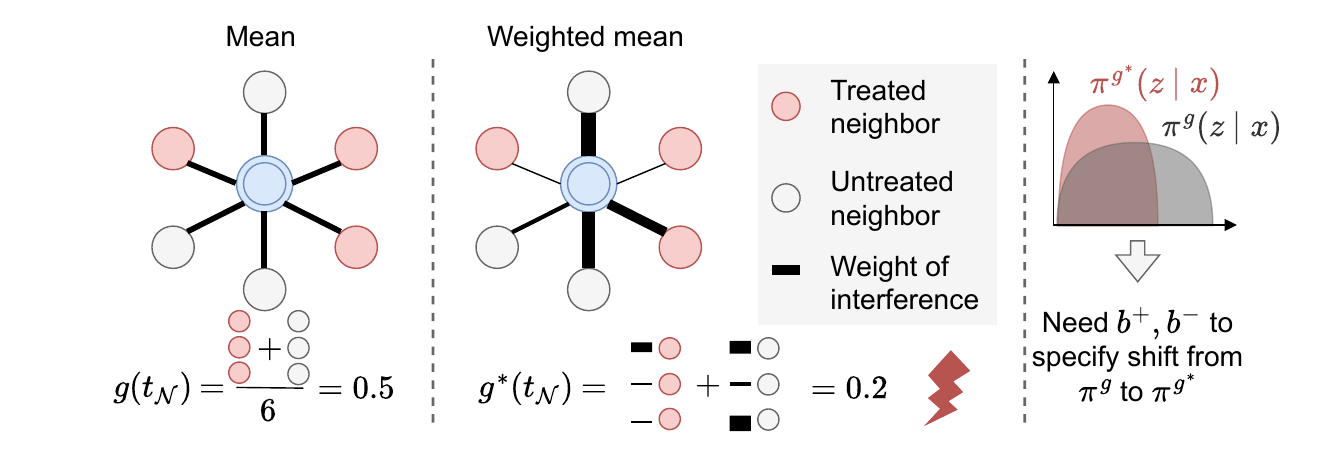}
    \vspace{-0.6cm}
    \caption{Exposure misspecification: weighted mean of treatments}
    \label{fig:weighted_mean_intuition}
\end{figure}
\vspace{-0.4cm}

\circled{2} \textbf{Thresholding:} Similarly, the mapping $g$ can be assumed to be a function of the sum or the proportion of treated neighbors \citep[e.g.,][]{Aronow.2017, McNealis.2024, Ogburn.2024, Qu.2024}. For example, the mapping could result in a binary variable $Z$ indicating if more than half of the neighbors were treated. 
\vspace{-0.5cm}
\begin{figure}[h]
    \centering
    \includegraphics[width=1\linewidth]{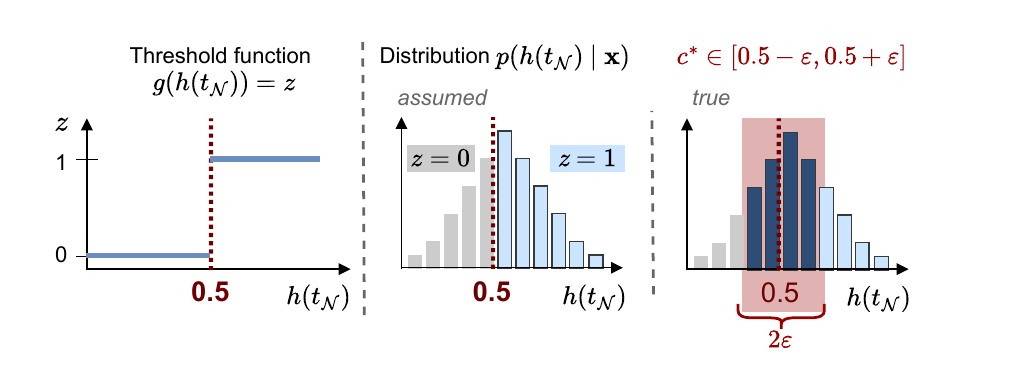}
    \vspace{-0.8cm}
    \caption{Exposure misspecification: thresholding function}
    \label{fig:thresholding_intuition}
\end{figure}

\vspace{-0.4cm}
\circled{3} \textbf{Higher-order spillovers:} The assumption that only the treatment of direct neighbors results in spillover effects can be too strong \citep[e.g.,][]{Belloni.2022, Leung.2022, Ogburn.2024, Weinstein.2023}. \emph{Higher-order neighbors} without a direct connecting edge in the network  potentially confound the causal relationship as well.
\vspace{-0.7cm}
\begin{figure}[h]
    \centering
    \includegraphics[width=1\linewidth]{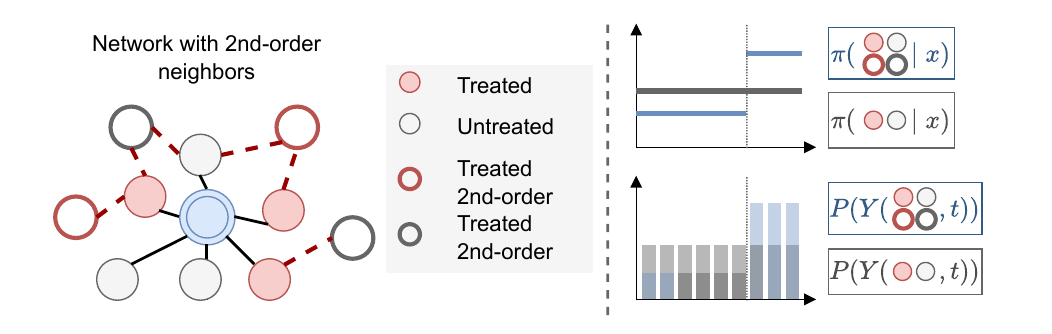}
    \vspace{-0.6cm}
    \caption{Exposure misspecification: higher-order spillovers}
    \label{fig:placeholder}
\end{figure}

\vspace{-0.5cm}
\subsection{Causal inference on networks}

We are interested in estimating the \emph{average potential outcome} (APO) under individual and neighborhood treatments $T=t$ and $Z=z$, where $z = g(t_{\mathcal{N}})$, given by 
\begin{align}
    \psi(t,z) &:= \mathbb{E}[Y(t,z)],
\end{align}
and the \emph{conditional average potential outcome} (CAPO) 
\begin{align}
    \mu(t,z,\mathbf{x}) &:= \mathbb{E}[Y(t,z) \mid \mathbf{X}=\mathbf{x}].
\end{align}
 
\vspace{-0.2cm}
The \emph{overall effect} can be decomposed into a \emph{direct effect} (capturing the impact of a unit’s own treatment on its outcome) and a \emph{spillover effect} (capturing the indirect impact of neighbors’ treatments on that unit’s outcome).

\begin{definition}[Direct effects (ADE / IDE)]
    \emph{The average (ADE) and individual (IDE) direct effects between individual treatment assignments $T=t$ and $T=t^{'}$ while keeping the neighborhood treatment $Z=z$ constant are defined as}
    
    \vspace{-0.5cm}
    \footnotesize
    \begin{align}
        \tau_d^{(t,z),(t^{'},z)} &:= \psi(t,z) - \psi(t^{'},z)\\
        \tau_{d_i}^{(t,z),(t^{'},z)}(\mathbf{x}_i) &:= \mu(t,z,\mathbf{x}_i) - \mu(t^{'},z, \mathbf{x}_i).
    \end{align}
    \normalsize
\end{definition}
\vspace{-0.2cm}
\begin{definition}[Spillover effects (ASE / ISE)]
    \emph{The average (ASE) and individual (ISE) spillover effects between neighborhood treatment $Z=z$ and $Z=z^{'}$ while keeping the individual treatment $T=t$ constant are defined as}

    \vspace{-0.5cm}
    \footnotesize
    \begin{align}
        \tau_s^{(t,z),(t,z^{'})} &:= \psi(t,z) - \psi(t,z^{'})\\
        \tau_{s_i}^{(t,z),(t,z^{'})}(\mathbf{x}_i) &:= \mu(t,z,\mathbf{x}_i) - \mu(t,z^{'}, \mathbf{x}_i).
    \end{align}
    \normalsize
    \vspace{-0.2cm}
\end{definition}
\vspace{-0.2cm}
\begin{definition}[Overall effects (AOE / IOE)]
    \emph{The average (AOE) and individual (IOE) total effects between individual treatment assignments $T=t$ and $T=t^{'}$ and neighborhood treatment assignments $Z=z$ and $Z=z^{'}$ are defined as}

    \vspace{-0.5cm}
    \footnotesize
    \begin{align}
        \tau_o^{(t,z),(t^{'},z^{'})} &:= \psi(t,z) - \psi(t^{'},z^{'})\\
        \tau_{o_i}^{(t,z),(t^{'},z^{'})}(\mathbf{x}_i) &:= \mu(t,z,\mathbf{x}_i) - \mu(t^{'},z^{'}, \mathbf{x}_i).
    \end{align}
    \normalsize
\end{definition}
\vspace{-0.2cm}
As standard in causal inference on networks \citep[e.g.,][]{Chen.2024b, Forastiere.2021}, we make the standard assumptions on consistency, unconfoundedness, and positivity, but adapted to the network interference setting.

\begin{assumption}[Network consistency]\label{assumption:consistency}
    The potential outcome equals the observed outcome given the same unit and neighborhood treatment exposure, i.e., $y_i = y_i(t_i,  t_{\mathcal{N}_i})$ if $i$ receives treatment $t_i$ and neighborhood treatment $t_{\mathcal{N}_i}$.   
\end{assumption}
\begin{assumption}[Network interference]\label{assumption:interference}
    A unit's treatment only affects its own as well as its neighbors' outcomes. The interference of the treatment with the neighbors outcomes is given by a (potentially unknown) summary function $g^{\ast}$, i.e., $\forall t_{\mathcal{N}_i},t^{'}_{\mathcal{N}_i}$ which satisfy $g^{\ast}(t_{\mathcal{N}_i}) = g^{\ast}(t^{'}_{\mathcal{N}_i})$, it holds $y_i(t_i,t_{\mathcal{N}_i}) = y_i(t_i,t^{'}_{\mathcal{N}_i})$.
\end{assumption}
\begin{assumption}[Network unconfoundedness]\label{assumption:unconfoundedness}
    Given the individual and the features of the neighborhood, the potential outcome is independent of the individual and the neighborhood treatment, i.e., $\forall t, t_{\mathcal{N}}: y_i(t, t_{\mathcal{N}}) \independent t_i, t_{\mathcal{N}_i} \mid \mathbf{x}_i$. If the summary function $g^{\ast}$ is correctly specified, it also holds $\forall t, g^{\ast}(t_{\mathcal{N}}): y_i(t, g^{\ast}(t_{\mathcal{N}})) \independent t_i, g^{\ast}(t_{\mathcal{N}_i}) \mid \mathbf{x}_i$. 
\end{assumption}
\begin{assumption}[Network positivity]\label{assumption:positivity}
    Given the individual and neighbors’ features, every treatment pair $(t, z)$ is observed with a positive probability, i.e., $ 0 < p(t, z \mid \mathbf{x}) < 1$ for all $\mathbf{x}, t, z.$
\end{assumption}

Under Assumptions \ref{assumption:consistency}--\ref{assumption:positivity} and correctly specified exposure mapping $g^{\ast}$, the (conditional) potential outcomes can be identified from observational data. However, if the exposure mapping $g \neq g^{\ast}$ is misspecified, Assumptions \ref{assumption:interference} and \ref{assumption:unconfoundedness} are \textit{not} satisfied. Therefore, the potential outcomes is \textit{not} point-identified from the existing data.

\subsection{Objective: partial identification under misspecification of the exposure mapping}

We propose to move to partial identification and compute upper and lower bounds $\mu^{\pm}(t,z^{\ast}, \mathbf{x})$ on the potential outcomes and treatment effects under such misspecification.\footnote{In the main paper, we focus on bounds for potential outcomes. We provide extensions to the treatment effects in Supplement~\ref{sec:appendix_theory}.}

We formalize the partial identification as a distribution shift in the \emph{exposure mapping propensity} $\pi^g(z \mid \mathbf{x}) := p(g(t_{\mathcal{N}})=z \mid \mathbf{x})$ between the employed and true but unknown mapping $g$ and $g^{\ast}$. 
For given shifts between $g$ and $g^{\ast}$, we aim to construct upper and lower bounds $b^{-}(z,\mathbf{x}) \leq b^{+}(z,\mathbf{x})$ with $b^{-}(z,\mathbf{x}) \in (0,1]$ and $b^{+}(z,\mathbf{x}) \in [1,\infty)$ on the generalized propensity ratio, such that, for all $z,\mathbf{x}$, we have
\vspace{-0.1cm}
\begin{align}\label{eq:ratio_bound}
    b^{-}(z,\mathbf{x}) \leq \frac{p(g^{\ast}(t_{\mathcal{N}}) = z \mid \mathbf{x})}{p(g(t_{\mathcal{N}})=z \mid \mathbf{x})} \leq b^{+}(z,\mathbf{x}).
\end{align}
\vspace{-0.5cm}

Of note, we do not impose any parametric assumption on the data-generating process. The specific interpretation and construction of $b^{\pm}$ depends on the definition of $g$ and $g^{\ast}$. 

We now formalize how our framework quantifies misspecifications in order to obtain bounds for the different exposure mappings. We provide justifications in Supplement~\ref{sec:appendix_proofs}.

\circled{1} \emph{Weighted mean of neighborhood treatments:} Let the exposure mapping $g(t_{\mathcal{N}}) := \sum_{j \in \mathcal{N}}\frac{t_j}{n} = \frac{N_T}{n}$ be specified as the proportion of treated neighbors, where $N_T$ denotes the number of treated neighbors and $n$ denotes the neighborhood size. We assume the true exposure mapping $g^{\ast}(t_{\mathcal{N}})$ is given by a weighted proportion of treated neighbors, where each weight is allowed to differ from $\frac{1}{n}$ for at most a value $\frac{1}{n} \geq \varepsilon \geq 0$. Then, the upper and lower bound are

\vspace{-0.5cm}
\footnotesize
\begin{align}
    b^-(z,\mathbf{x}) = \inf_{s \in \mathcal{Z}}\frac{P(\frac{ns}{1-\varepsilon n} \leq N_T \leq \frac{nz}{1+\varepsilon n}\mid \mathbf{x})}{P(ns\leq N_T \leq nz\mid \mathbf{x})}  ,
\end{align}
\vspace{-0.5cm}
\begin{align}
    b^+(z,\mathbf{x}) = \sup_{s \in \mathcal{Z}}\frac{P(\frac{ns}{1+\varepsilon n} \leq N_T \leq \frac{nz}{1-\varepsilon n}\mid \mathbf{x})}{P(ns\leq N_T \leq nz\mid \mathbf{x})}.
\end{align}
\normalsize

\vspace{-0.2cm}
\circled{2} \emph{Thresholding:} Let $h(t_{\mathcal{N}}) := \sum_{j \in \mathcal{N}}\frac{t_j}{n}$ denote the proportion of treated neighbors. Assume the exposure mapping is specified  through a threshold as $g(t_{\mathcal{N}}) =  f(h(t_{\mathcal{N}})) := \mathbf{1}_{[h(t_{\mathcal{N}}) \geq c]}$. Then, $P(g(t_{\mathcal{N}}) = 1 \mid \mathbf{x}) = P(N_T \geq nc \mid \mathbf{x})$, where $N_T$ denotes the number of treated neighbors. We now allow the true threshold $c^{\ast}$ defining $g^{\ast}(t_{\mathcal{N}}) := \mathbf{1}_{[h(t_{\mathcal{N}}) \geq c^{\ast}]}$ to differ by an amount $\varepsilon \in [0,\min\{c, 1-c\}]$ from $c$, i.e., $c^{\ast} \in [c \pm \varepsilon]$. By straightforward computation, we receive

\vspace{-0.5cm}
\footnotesize
\begin{align}
    \frac{P(N_T \geq n(c+\varepsilon) \mid \mathbf{x})}{P(P(N_T \geq nc \mid \mathbf{x})} &\leq \frac{P(g^{\ast}(t_{\mathcal{N}}) = 1 \mid \mathbf{x})}{P(g(t_{\mathcal{N}})=1 \mid \mathbf{x})}\\ &\leq \frac{P(N_T \geq n(c-\varepsilon) \mid \mathbf{x})}{P(P(N_T \geq nc \mid \mathbf{x})},
\end{align}

\vspace{-0.4cm}
\normalsize
where the bounds for $z=0$ follow with the complement probabilities.

\circled{3} \emph{Higher-order spillovers:} Here, $g$ is misspecified in that it is not merely a function of $t_{\mathcal{N}_i}$, but also a function of other treatments $t_U \subset t_{\mathcal{N}_{-i}}$ for the respective node $i$. As a result, $t_U$ biases the exposure summary $z$ in an unobserved manner. We thus apply sensitivity bounds from the unobserved confounding literature \cite{Dorn.2022, Frauen.2023b}, in that we require user-specified functions $b^{\pm}$, such that
\vspace{-0.3cm}
\begin{align}
    b^{-}(z,\mathbf{x}) \leq \frac{p(g(t_{\mathcal{N} \cup U}) = z \mid \mathbf{x})}{p(g(t_{\mathcal{N}})=z \mid \mathbf{x})} \leq b^{+}(z,\mathbf{x}),
\end{align}

\vspace{-0.3cm}
where $g$ can be any exposure mapping, such as the mean or the thresholding function in \circled{1} and \circled{2}.


\section{Our partial identification framework}
\label{sec:method}

We now present our framework. We derive sharp and valid bounds $\mu^{\pm}(t,z,\mathbf{x})$ on the potential outcomes following ideas from causal sensitivity analysis (Section~\ref{sec:bounds}). In Supplement~\ref{sec:appendix_theory}, we translate these into corresponding bounds for the direct, spillover, and overall effects. Next, we develop an orthogonal estimator $\hat{\mu}_\text{ortho}^{\pm}(t,z,\mathbf{x})$ based on orthogonal statistical learning theory (Section~\ref{sec:method_robust}). Finally, we derive the theoretical properties of our estimator, including convergence rates, sharpness, and validity guarantees (Section~\ref{sec:method_properties}). All proofs are in Supplement~\ref{sec:appendix_proofs}.

\subsection{Derivation of the bounds $\mu^{\pm}(t,z,\mathbf{x})$}
\label{sec:bounds}

We now introduce our \emph{sharp} upper and lower bounds of the CAPO with respect to the misspecification $b^\pm$.

\begin{definition}\label{def:sharp_bounds}
    \emph{Let $\Tilde{\mathbb{P}}$ denote a distribution on $(\mathbf{X}, T, Z, Y(T,Z))$, such that (i)~$\Tilde{\mathbb{P}}$ matches the observed distribution $\mathbb{P}$ on $(\mathbf{X}, T, Z, Y)$, and (ii)~the corresponding conditional distribution $\Tilde{\pi}^g(z\mid \mathbf{x})$ satisfies
    $b^{-}(z,\mathbf{x}) \leq \frac{\Tilde{\pi}^g(z \mid \mathbf{x})}{\pi^g(z \mid \mathbf{x})} \leq b^{+}(z,\mathbf{x})$ almost surely. Let $\mathcal{M}$ denote the set of such distributions $\Tilde{P}$. Then, the {sharp bounds} of the CAPO with respect to the misspecification bounds $b^{\pm}(z,\mathbf{x})$ are given by}
    \begin{align}
        &\mu^+(t,z,\mathbf{x}) = \sup_{\Tilde{P} \in \mathcal{M}} \mathbb{E}_{\Tilde{P}}[Y(t,z)\mid \mathbf{X=x}],\\
        &\mu^-(t,z,\mathbf{x}) = \inf_{\Tilde{P} \in \mathcal{M}} \mathbb{E}_{\Tilde{P}}[Y(t,z)\mid \mathbf{X=x}].
    \end{align}
    \vspace{-0.5cm}
\end{definition}

\textbf{Intuition:}
To obtain the CAPO bounds, we need to bound $\mathbb{E}[Y \mid t, z,\mathbf{x}] = \int_{\mathcal{Y}}yp(y\mid t, z,\mathbf{x})\,\mathrm{d}y.$ We can construct \emph{valid} bounds based on Eq.~\ref{eq:ratio_bound} by simply setting $\mu^{\pm}(y\mid t, z,\mathbf{x}) = \frac{1}{b^{\mp}(z,\mathbf{x})}\mathbb{E}[Y \mid t, z,\mathbf{x}]$. However, the resulting bounds are \emph{not sharp}, but these are conservative and potentially uninformative. To obtain the equalities in Definition~\ref{def:sharp_bounds}, we follow ideas from sensitivity analysis \cite{Dorn.2025, Frauen.2023b} and find a \emph{cut-off value} $C$, such that, in a very simplified notation, we have

\vspace{-0.5cm}
\footnotesize
\begin{align}
    \mu^{\pm}= \frac{1}{b^{\mp}}\int_{-\infty}^{C^{\pm}}yp(y\mid \cdot)\,\mathrm{d}y
    + \frac{1}{b^{\pm}}\int_{C^{\pm}}^{\infty}yp(y\mid \cdot)\,\mathrm{d}y.
\end{align}
\normalsize

\vspace{-0.2cm}
Let $F_Y(y):=F_Y(y \mid t,z,\mathbf{x})$ denote the conditional cumulative distribution function (CDF) of $Y$. We define the \emph{conditional quantile function} of the outcome at level $\alpha^{\pm} = \frac{(1-b^{\mp}(z,\mathbf{x}))b^{\pm}(z,\mathbf{x})}{b^{\pm}(z,\mathbf{x}) - b^{\mp}(z,\mathbf{x})}$ as

\vspace{-0.5cm}
\footnotesize
\begin{equation}
\begin{aligned} \label{eq:quantile_func}
    Q^{\pm}(t,z,\mathbf{x}) := \begin{cases}
    \inf  \Big\{y \mid F_Y(y), \geq
    \alpha^{\pm} \Big\} , \; &\text{if}\; b^-<1<b^+,\\
    \inf  \Big\{y \mid F_Y(y), \geq
    \frac{1}{2}\Big\} , \; &\text{if} \; b^-=b^+,
    \end{cases}
\end{aligned}
\end{equation}
\normalsize
where we abbreviated the notation for $b^{\pm}(z,\mathbf{x})$ through $b^{\pm}$.

The bounds are \emph{sharp}  under interference by employing $Q^{\pm}(t,z,\mathbf{x})$ as the cut-off value $C$. We formalize this in the following theorem, where we present the closed-form solution that facilitates estimation:

\begin{restatable}{theorem}{bounds}
\label{thm:plug_in_bounds}
    Let $Q^{\pm}(t,z,\mathbf{x})$ be defined as in Eq.~(\ref{eq:quantile_func}) and let $(u)_+ = \max\{u,0\}$. The sharp CAPO upper and lower bounds are given by

    \vspace{-0.5cm}
    \footnotesize
    \begin{equation}    
    \begin{aligned}
        \mu^{\pm}(t,z, \mathbf{x}) &= Q^{\pm}(t,z,\mathbf{x})\\&+ \frac{1}{b^{\mp}(z,\mathbf{x})}\mathbb{E}\bigl[(Y-Q^{\pm}(t,z,\mathbf{x}))_{+}\mid t,z,\mathbf{x}\bigr]\\ &- \frac{1}{b^{\pm}(z,\mathbf{x})}\mathbb{E}\bigl[ (Q^{\pm}(t,z,\mathbf{x})-Y)_{+} \mid t,z,\mathbf{x} \bigr].
    \end{aligned}
    \end{equation}
    \normalsize
\end{restatable}

\begin{remark}[Limits of the sensitivity model]
\label{rem:sensitivity_limits}
If $b^{-}(z,\mathbf{x})=b^{+}(z,\mathbf{x})=1$ (no exposure-mapping shift), then the identified set collapses and $\mu^{\pm}(t,z,\mathbf{x})=\mathbb{E}[Y\mid t, z, \mathbf{x}]$.
As $b^{+}(z,\mathbf{x})\to\infty$ with $b^{-}(z,\mathbf{x})$ fixed, the upper bound concentrates on the top $b^{-}(z,x)$-tail of $Y\mid(t,z,\mathbf{x})$ (the lower bound on the bottom $b^{-}(z,x)$-tail). In the extreme limit $b^{-}(z,\mathbf{x})\to 0$ and $b^{+}(z,\mathbf{x})\to\infty$, the bounds become vacuous and converge to the conditional support 
$
\mu^{+}(t,z,\mathbf{x})\to \esssup\bigl(Y\mid t,z, \mathbf{x}\bigr)
$; $
\mu^{-}(t,z,\mathbf{x})\to \essinf\bigl(Y\mid t,z,\mathbf{x}\bigr).
$\footnote{The essential supremum (essential infimum) is abbreviated by $\esssup$ ($\essinf$) and defines the supremum (infimum) of the essential upper (lower) bound of the conditional distribution $Y\mid t,z,\mathbf{x}$, i.e., the upper (lower) bound with non-zero measure.}
\end{remark}

\subsection{Orthogonal estimatior}
\label{sec:method_robust}

$\bullet$\,\textbf{Disadvantages of plug-in estimation.}
The characterization in Theorem~\ref{thm:plug_in_bounds} immediately suggests a \emph{plug-in} estimation strategy: estimate the cut-off $Q^{\pm}(t,z,\mathbf{x})$ and the two conditional moment functions

\vspace{-0.4cm}
\footnotesize
\begin{align*}
\gamma^{\pm}_u(t,z,\mathbf{x})
&:= \mathbb{E}\!\left[(Y-Q^{\pm}(\cdot))_{+}\mid T=t,Z=z,\mathbf{X}=\mathbf{x}\right],\nonumber\\
\gamma^{\pm}_l(t,z,\mathbf{x})
&:= \mathbb{E}\!\left[(Q^{\pm}(\cdot)-Y)_{+}\mid T=t,Z=z,\mathbf{X}=\mathbf{x}\right].
\end{align*} 
\normalsize

\vspace{-0.4cm}
Then, we obtain $\hat\mu^{\pm}(t,z,\mathbf{x})$ by substituting $(\hat Q^{\pm},\hat\gamma^{\pm}_u,\hat\gamma^{\pm}_l)$ into Thm.~\ref{thm:plug_in_bounds}. However, such plug-in estimators suffer from substantial finite-sample bias due to nuisance estimation error, especially when the nuisance functions are more complex than the bound function itself \cite{Kennedy.2019}. We therefore apply orthogonalization strategies \citep{Dorn.2025, Oprescu.2023} and derive orthogonal pseudo-outcomes for the bounds to then estimate $\mu^{\pm}(t,z,\mathbf{x})$ by regressing the pseudo-outcomes on $\mathbf{X}$. 

$\bullet$\,\textbf{Orthogonal pseudo-outcome.} Recall that $T\in\{0,1\}$ while $Z$ is discrete or continuous. This is relevant because our bounds involve evaluation at a fixed neighborhood exposure level $z$. When $Z$ is \emph{binary or discrete}, the functional $\mathbb{P}\mapsto \mu^{\pm}(t,z,\mathbf{x})$ is regular (pathwise differentiable) for each fixed $(t,z)$, so we can construct an efficient influence-function-based pseudo-outcome. When $Z$ is \emph{continuous}, evaluation at $Z=z$ is not path-wise differentiable; we therefore replace point evaluation by a locally smoothed target using a kernel $K_h(Z-z)$, which introduces a nonparametric bias-variance tradeoff governed by the bandwidth $h$.

For ease of presentation, we now present the orthogonal upper bound $\mu^{+}$ and defer the lower bound $\mu^{-}$ to Supplement~\ref{sec:appendix_theory}. We refer to $(\pi^t,\pi^g,Q^{\pm},\gamma^{\pm}_u,\gamma^{\pm}_l)$ as \emph{nuisances}.

\begin{restatable}{theorem}{orthobounds}
\label{thm:orthogonal_bounds}
Let $S=(\mathbf{X},Y,T,Z)$. Fix $(t,z)$. Define 
\vspace{-0.2cm}
\begin{equation*}
\omega_{z,h}(Z):=
\begin{cases}
\mathbf{1}_{[Z=z]}, & \text{if } Z\ \text{binary/discrete},\\
K_h(Z-z), & \text{if } Z\ \text{continuous},
\end{cases}\vspace{-0.2cm}
\end{equation*}
and let $\pi^g(Z\mid\mathbf{X})$ denote the conditional pmf (discrete $Z$) or density (continuous $Z$).
Let $\widehat\eta=(\widehat\pi^t,\widehat\pi^g,\widehat Q^+,\widehat\gamma_u^+,\widehat\gamma_l^+)$ be a set of estimated nuisances. 
Then, an orthogonal pseudo-outcome for the CAPO upper bound $\mu^+(t,z,\mathbf{x})$ is:
\vspace{-0.2cm}
\footnotesize
\begin{align}
& \phi^+_{t,z}(S;\widehat\eta) = \nonumber \\
&  \widehat Q^+(t,z,\mathbf{X}) +\frac{\widehat\gamma_u^+(t,z,\mathbf{X})}{b^{-}(z,\mathbf{X})}
-\frac{\widehat\gamma_l^+(t,z,\mathbf{X}) }{b^{+}(z,\mathbf{X})} \label{eq:pseudo_outcome_unified}
 \\
& +\frac{\mathbf{1}_{[T=t]}\,\omega_{z,h}(Z)}
{\widehat\pi^t(\mathbf{X})\,\widehat\pi^g(Z\mid\mathbf{X})}
\Bigg[
\frac{(Y-\widehat Q^+(t,Z,\mathbf{X}))_{+}-\widehat\gamma_u^+(t,Z,\mathbf{X})}{b^{-}(Z,\mathbf{X})} \nonumber \\
&\hspace{2.8cm}
-\frac{(\widehat Q^+(t,Z,\mathbf{X})-Y)_{+}-\widehat\gamma_l^+(t,Z,\mathbf{X})}{b^{+}(Z,\mathbf{X})}
\Bigg]. 
\nonumber
\end{align}
\normalsize
Moreover, when $\widehat\eta=\eta$, the pseudo-outcome is unbiased for its target bound functional (see Remark~\ref{rem:unbiasedness_phi}).
\end{restatable} 

\begin{remark}[Unbiasedness of the pseudo-outcome]
\label{rem:unbiasedness_phi}
When $\widehat\eta=\eta$ and $Z$ is \emph{discrete}, we have
$\mathbb{E}\!\left[\phi^+_{t,z}(S;\eta)\mid \mathbf{X}=\mathbf{x}\right]\!=\!\mu^+(t,z,\mathbf{x})$ and
$\mathbb{E}\!\left[\phi^+_{t,z}(S;\eta)\right]\!=\!\psi^+(t,z).$
When $Z$ is \emph{continuous}, the kernel-localized pseudo-outcome targets a bandwidth-indexed functional $(\mu_h^+,\psi_h^+)$; under standard smoothness in $z$, $\mu_h^+(t,z,\mathbf{x})\to\mu^+(t,z,\mathbf{x})$ and $\psi_h^+(t,z)\to\psi^+(t,z)$ as $h\downarrow 0$.
\end{remark}

$\bullet$\,\textbf{Bound estimation algorithm.}
Motivated by Theorem~\ref{thm:orthogonal_bounds}, we estimate the bounds via a \emph{two-stage procedure} (Algorithm~\ref{alg:ortho_bound_algorithm}): we first learn the nuisance functions $\widehat\eta$, then evaluate the orthogonal pseudo-outcome $\phi^+_{t,z}(S;\widehat\eta)$ and finally obtain (i)~the CAPO bound $\widehat\mu^+(t,z,\mathbf{x})=\widehat{\mathbb{E}}_n[\phi^+_{t,z}(S;\widehat\eta)\mid \mathbf{X}=\mathbf{x}]$ by regressing the pseudo-outcome on $\mathbf{X}$ and (ii)~the APO bound $\widehat\psi^+(t,z)=\widehat{\mathbb{E}}_n\!\left[\phi^+_{t,z}(S;\widehat\eta)\right]$ via sample averaging.
To mitigate overfitting bias and enable standard orthogonalization guarantees, we compute $\widehat\phi^+_{t,z,i}$ using $K$-fold cross-fitting \citep{Chernozhukov.2018}: each $\widehat\phi^+_{t,z,i}$ uses nuisance estimates trained on data not containing $i$. 

\begin{algorithm}[t]
\caption{Orthogonal estimator for the bounds}
\label{alg:ortho_bound_algorithm}
\begin{algorithmic}[1]
\scriptsize
\STATE \textbf{Input:} data $\{S_i=(\mathbf{X}_i,Y_i,T_i,Z_i)\}_{i=1}^n$, target $(t,z)$, bandwidth $h$ (if $Z$ continuous), folds $\{\mathcal{I}_k\}_{k=1}^K$, nuisance estimators, regression learner $\widehat{\mathbb{E}}_n$
\FOR{$k=1,\dots,K$}
    \STATE Fit nuisances $\widehat\eta^{(-k)}$ on $\{S_i: i\notin\mathcal{I}_k\}$
    \FOR{$i\in\mathcal{I}_k$}
        \STATE $\widehat\phi^+_{t,z,i}\leftarrow \phi^+_{t,z}(S_i;\widehat\eta^{(-k)})$
    \ENDFOR
\ENDFOR
\STATE $\widehat\psi^+(t,z)\leftarrow \frac{1}{n}\sum_{i=1}^n \widehat\phi^+_{t,z,i}$ \hfill (sample average)
\STATE $\widehat\mu^+(t,z,\mathbf{x})\leftarrow \widehat{\mathbb{E}}_n[\phi^+_{t,z}(S;\widehat\eta)\mid \mathbf{X}=\mathbf{x}]$ \hfill (regression fit)
\STATE \textbf{Output:} $\widehat\mu^+(t,z,\cdot)$, $\widehat\psi^+(t,z)$ 
\end{algorithmic}
\end{algorithm}
\normalsize

\subsection{Theoretical properties of our bound estimator}
\label{sec:method_properties}

Theorem~\ref{thm:plug_in_bounds} shows that the identified CAPO bounds $\mu^\pm(t,z,\mathbf{x})$ are sharp. We establish \emph{three additional guarantees} for our orthogonal estimator (Theorem~\ref{thm:orthogonal_bounds}).
\circledgreen{1}~~Orthogonality yields \emph{second-order} sensitivity to nuisance estimation error, implying quasi-oracle rates for the CAPO bounds and (for discrete $Z$) root-$n$ inference for the APO bounds.
\circledgreen{2}~~If $Q^\pm$ is consistently estimated and either the propensity models $(\pi^t,\pi^g)$ \emph{or} the moment functions $(\gamma_u^\pm,\gamma_l^\pm)$ are consistently estimated, then our estimated endpoints converge (in $L_2(P_{\mathbf{X}})$) to the \emph{sharp} bounds.
\circledgreen{3}~~If $Q^\pm$ is misspecified, but either $(\pi^t,\pi^g)$ or $(\gamma_u^\pm,\gamma_l^\pm)$ is consistently estimated, the (C)APO intervals remain asymptotically \emph{valid}, though potentially conservative.
We present results for discrete $Z$ in the main text; the case with continuous $Z$ is deferred to Supplement~\ref{sec:appendix_continuousZ}. All proofs are in Supplement~\ref{sec:appendix_proofs}.

$\bullet$\,\textbf{\circledgreen{1}~Quasi-oracle learning via orthogonality.}
The next theorem is the key guarantee: under standard regularity conditions, orthogonality implies that nuisance errors contribute only through error
\emph{products}.\footnote{\emph{Notation:} Let $\|f\|_2:=\{\mathbb{E}[f(\mathbf{X})^2]\}^{1/2}$ denote the $L_2(P_{\mathbf{X}})$ norm. 
We write $W_n=o_p(a_n)$ if $W_n/a_n\to 0$ in probability and $W_n=O_p(a_n)$ if $W_n/a_n$ is bounded in probability.We write $\rightsquigarrow$ for convergence in distribution.}

\begin{assumption}[Regularity and overlap]
\label{assump:overlap_bounded}
There exist $\varepsilon>0$ and $M<\infty$ such that, a.s.:
(i) $\varepsilon \le \pi^t(\mathbf{X}),\widehat \pi^t(\mathbf{X}) \le 1-\varepsilon$;
(ii) if $Z$ is discrete, $\varepsilon \le \pi^g(z\mid \mathbf{X}),\widehat\pi^g(z\mid \mathbf{X})$ for all relevant $(z,\mathbf{X})$;
if $Z$ is continuous, there exists a neighborhood $\mathcal N_z$ of $z$ such that for all $u\in\mathcal N_z$,
$\varepsilon \le \pi^g(u\mid \mathbf{X}),\widehat\pi^g(u\mid \mathbf{X}) \le M$;
(iii) $|Y|,|\widehat\gamma_u^\pm|,|\widehat\gamma_l^\pm|,|\widehat Q^\pm| \le M$.
\end{assumption}

\begin{restatable}[Second-order nuisance error (discrete $Z$)]{theorem}{nuisanceReminder}
\label{thm:second_order_remainder}
Assume $Z$ is discrete and Assumption~\ref{assump:overlap_bounded} holds. Let $\widehat\eta=(\widehat\pi^t,\widehat\pi^g,\widehat Q^+,\widehat\gamma_u^+,\widehat\gamma_l^+)$ be the cross-fitted nuisances used in $\phi^+_{t,z}(S;\widehat\eta)$ from Theorem~\ref{thm:orthogonal_bounds}. Define $r_{n,\pi}:=\|\widehat\pi^t-\pi^t\|_2+\|\widehat\pi^g-\pi^g\|_2$, $r_{n,Q}:=\|\widehat Q^+-Q^+\|_2$, and $r_{n,\gamma}:=\|\widehat\gamma_u^+-\gamma_u(\widehat Q^+;\cdot)\|_2+\|\widehat\gamma_l^+-\gamma_l(\widehat Q^+;\cdot)\|_2$, where $\gamma_u(\widehat Q^+;\mathbf X):=\mathbb{E}[(Y-\widehat Q^+(\mathbf X))_+\mid T=t,Z=z,\mathbf X]$ and $\gamma_l(\widehat Q^+;\mathbf X):=\mathbb{E}[(\widehat Q^+(\mathbf X)-Y)_+\mid T=t,Z=z,\mathbf X]$. Then

\vspace{-0.4cm}
\footnotesize
\begin{equation}\label{eq:remainder_bound}
\left\|\mathbb{E}\!\left[\phi^+_{t,z}(S;\widehat\eta)-\phi^+_{t,z}(S;\eta)\mid \mathbf X\right]\right\|_2
=O_p\!\left(r_{n,\pi}\,r_{n,\gamma}+r_{n,Q}^2\right).
\end{equation}
\normalsize
\vspace{-0.4cm}
\end{restatable}

\vspace{-0.2cm}
Thus, the contribution of nuisances to the estimation error is only second order. Next, we show that the final-stage regression can achieve the same rate as if the true pseudo-outcomes were observed (a quasi-oracle property), even if the nuisances converge more slowly. 

\begin{assumption}[Second-stage regression rate]
\label{assump:second_stage}
Fix $(t,z)$ and let $\widehat\phi^+_{t,z,i}$ denote the cross-fitted pseudo-outcome.
Let $m^+_{t,z}(\mathbf{x}):=\mathbb{E}[\widehat\phi^+_{t,z}\mid \mathbf{X}=\mathbf{x}]$.
Assume the regression learner used to form $\widehat\mu^+(t,z,\cdot)$ satisfies
\begin{equation}
    \|\widehat\mu^+(t,z,\cdot)-m^+_{t,z}(\cdot)\|_2 = O_p(\delta_n),
\end{equation}
for some (possibly model-dependent) rate $\delta_n$.
\end{assumption}

\begin{remark}[Second-stage regression assumption]
Assumption~\ref{assump:second_stage} treats the final-stage regression step as a black box: it assumes that, when regressing
the cross-fitted pseudo-outcomes on $\mathbf{X}$, the learner attains an $L_2$ error rate $\delta_n$ uniformly over the admissible nuisance estimates $\widehat\eta\in\Xi$. A broad class of learners satisfy this, including nonparametric least-squares/ERM estimators over a bounded function class $\mathcal{F}$ with bracketing entropy $\log N_{[]}(\mathcal{F},\epsilon)\lesssim \epsilon^{-r}$ ($0<r<2$), which yields the
usual regression rate $\delta_n \asymp n^{-1/(2+r)}$ (up to approximation error); in particular, for $d$-dimensional H\"older$(\beta)$ classes, $\delta_n = n^{-\beta/(2\beta+d)}$. More generally, black-box regressors satisfying standard stability/oracle-inequality properties (e.g., linear smoothers) also fit this template \citep{Kennedy.2023b}. We therefore state our results in terms of $\delta_n$, which separates orthogonalization from the choice of final-stage regression method.
\end{remark}

\begin{restatable}[Quasi-oracle rates and inference (discrete $Z$)]{corollary}{quasiOracleRates}
\label{cor:quasi_oracle_and_inference}
Suppose Assumptions~\ref{assump:overlap_bounded} and \ref{assump:second_stage} hold, and let
$r_{n,\pi}, r_{n,\gamma}, r_{n,Q}$ be as in Theorem~\ref{thm:second_order_remainder}.

\underline{CAPO rates:} The CAPO upper-bound estimator satisfies
\begin{equation*}
\label{eq:capo_quasi_oracle_rate}
\|\widehat\mu^+(t,z,\cdot)-\mu^+(t,z,\cdot)\|_2
=O_p\!\left(\delta_n + r_{n,\pi}\,r_{n,\gamma}+r_{n,Q}^2\right).
\end{equation*}
In particular, if $r_{n,\pi}\,r_{n,\gamma}+r_{n,Q}^2=o_p(\delta_n)$, then
$\|\widehat\mu^+(t,z,\cdot)-\mu^+(t,z,\cdot)\|_2=O_p(\delta_n)$.

\underline{APO rates:} The APO upper-bound estimator $\widehat\psi^+(t,z)=\mathbb{E}_n[\widehat\phi^+_{t,z}]$ satisfies

\vspace{-0.5cm}
\footnotesize
\begin{equation}
\label{eq:apo_asymp_linear}
\left|\widehat\psi^+(t,z)\!-\!\psi^+(t,z)\right|
\!=\!
O_p\left(n^{-1/2}+ r_{n,\pi}\,r_{n,\gamma}+r_{n,Q}^2\right).
\end{equation}
\normalsize
If moreover $r_{n,\pi}\,r_{n,\gamma}+r_{n,Q}^2=o_p(n^{-1/2})$, then
\begin{equation}
\label{eq:apo_clt}
\sqrt{n}\left(\widehat\psi^+(t,z)-\psi^+(t,z)\right)\ \rightsquigarrow\
\mathcal{N}\!\left(0,\,V^+(t,z)\right),
\end{equation}
i.e., the APO bound estimator is asymptotically normal with variance $V^+(t,z):=\mathrm{Var}\!(\phi^+_{t,z}(S;\eta))$ (efficiency bound).
\end{restatable}

Corollary~\ref{cor:quasi_oracle_and_inference} establishes a \emph{quasi-oracle} property: if the nuisance estimators converge at rate $o_p(\delta_n^{1/2})$ for CAPO bounds (or $o_p(n^{-1/4})$ for APO), the estimator achieves the oracle rate $O_p(\delta_n)$, as if the nuisances were known. This follows, since nuisance errors enter only through the second-order remainder $r_{n,\pi}r_{n,\gamma}+r_{n,Q}^2$. For APOs, this additionally enables valid and tight inference.

$\bullet$\,\textbf{\circledgreen{2}~Sharpness of the \textit{estimated} bounds.}
Next, we state conditions under which our estimates converge to the \emph{sharp} identified bounds from Theorem~\ref{thm:plug_in_bounds}, so that the \textit{estimated} bounds are also sharp.

\begin{restatable}[Consistency for sharp bounds (discrete $Z$)]{proposition}{sharpness}
\label{prop:sharpness}
Assume the conditions of Corollary~\ref{cor:quasi_oracle_and_inference} hold. Suppose $\delta_n=o_p(1)$ and
$r_{n,Q}=o_p(1)$, and, in addition, either $r_{n,\pi}=o_p(1)$ or $r_{n,\gamma}=o_p(1)$.
Then,
$\|\widehat\mu^\pm(t,z,\cdot)-\mu^\pm(t,z,\cdot)\|_2 =o_p(1)$ and $|\widehat\psi^\pm(t,z)-\psi^\pm(t,z)|=o_p(1) $,
Consequently, the estimated CAPO and APO intervals converge to the \emph{sharp} identified intervals.
\end{restatable}
Proposition~\ref{prop:sharpness} shows that, if $\widehat Q^\pm$ is consistent and either the propensity or conditional-moment nuisance is consistent, then Algorithm~\ref{alg:ortho_bound_algorithm} consistently estimates the sharp CAPO/APO bounds in Theorem~\ref{thm:plug_in_bounds}. 

$\bullet$\,\textbf{\circledgreen{3}~Validity under $\widehat Q^\pm$ misspecification.}
Sharpness guarantees require that $Q^\pm$ is estimated consistently. We show that, even when $Q^\pm$ is misspecified, our estimator yields conservative but valid bounds, provided one of the two (first-stage) nuisance ``blocks'' is consistently learned.

\begin{restatable}[Asymptotic validity under misspecified cutoffs (discrete $Z$)]{corollary}{validity}
\label{cor:validity}
Assume the conditions of Corollary~\ref{cor:quasi_oracle_and_inference} hold. Let $\overline Q^\pm(t,z,\mathbf{x})$ be any measurable cut-off and define
the induced (possibly non-sharp) bounds

\vspace{-0.5cm}
\footnotesize
\begin{align}
\label{eq:mu_barQ}
    \overline \mu^{\pm}(t,z, \mathbf{x}; \overline Q^{\pm}) &= \overline Q^{\pm}(t,z,\mathbf{x}) \\
    &+ \frac{1}{b^{\mp}(z,\mathbf{x})}\mathbb{E}\bigl[(Y-\overline Q^{\pm}(t,z,\mathbf{x}))_{+}\mid t,z,\mathbf{x}\bigr] \nonumber\\ 
    &- \frac{1}{b^{\pm}(z,\mathbf{x})}\mathbb{E}\bigl[ (\overline Q^{\pm}(t,z,\mathbf{x})-Y)_{+} \mid t,z,\mathbf{x} \bigr].
    \nonumber
\end{align}
\normalsize
(and analogously $\overline\psi^\pm(t,z):=\mathbb{E}[\overline\mu^\pm(t,z,\mathbf{X})]$).
Then, $[\overline\mu^-(t,z,\mathbf{x}),\overline\mu^+(t,z,\mathbf{x})]$ is a \underline{valid} (not necessarily sharp) CAPO interval, and likewise for $[\overline\psi^-(t,z),\overline\psi^+(t,z)]$.

Moreover, if $\widehat Q^\pm \to \overline Q^\pm$ in $L_2$ and either (i)~$(\widehat\pi^t,\widehat\pi^g)$ is consistent, or
(ii)~$(\widehat\gamma_u^\pm,\widehat\gamma_l^\pm)$ is consistent for the tail-moment targets induced by $\overline Q^\pm$,
then the resulting estimated (C)APO intervals converge to $[\overline\mu^-,\overline\mu^+]$ and $[\overline\psi^-,\overline\psi^+]$ and are asymptotically valid, though potentially conservative. If $\overline Q^\pm=Q^\pm$, the bounds coincide with the sharp bounds.
\end{restatable}

\begin{remark}[Continuous $Z$]
\label{rem:continuous_z_rates}
When $Z$ is continuous, evaluation at a point $z$ requires kernel localization, leading to the usual
bias-variance tradeoff in the bandwidth. We defer the corresponding rates and inference results to
Supplement~\ref{sec:appendix_continuousZ}.
\end{remark}

\vspace{-0.3cm}
\section{Experiments}
\label{sec:experiments}

\textbf{Data:}\footnote{Data and implementation details are in Supplement~\ref{sec:appendix_implementation}.}
We follow common practice in causal partial identification and evaluate our framework on synthetic datasets. We generate simple networks with $N=1000$ nodes and a $1$-dimensional covariate (small dataset) and more complex networks with $N=6000$ nodes and $6$-dimensional covariates (large dataset).

\textbf{Evaluation:}
Our goal is to \emph{demonstrate the theoretical properties} of our framework: (1)~We evaluate our bounds in terms of \textit{validity}, i.e., we show that our bounds contain the true outcome whenever the constraints given by $b^{\pm}$ are satisfied. (2)~We compare the \emph{convergence} of our orthogonal bound estimator against the plug-in estimator. (3)~We assess the \emph{informativeness} of our bounds in terms of the widths of the resulting intervals. We report all results over 10 runs.\footnote{Our goal is to show the validity and advantages of our model-agnostic partial identification framework. We thus refrain from comparing different instantiations of the nuisances or second-stage models.}

\begin{figure}[h]
    \vspace{-0.2cm}
    \centering
    \includegraphics[width=1\linewidth]{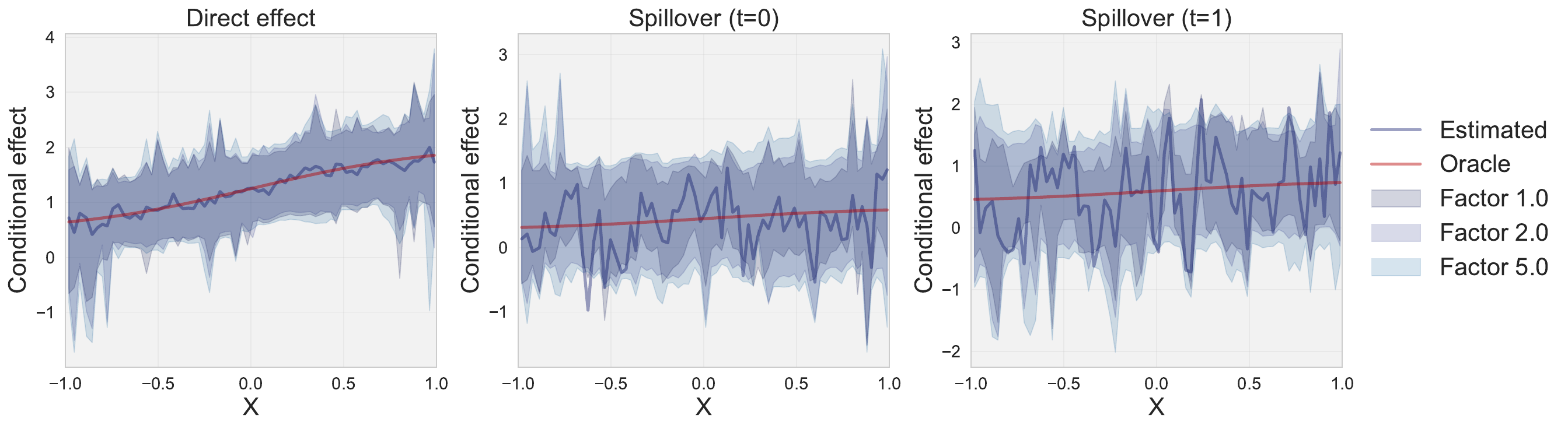}
    \vspace{-0.6cm}
    \caption{\textbf{Conditional effect bounds:} Visualization of our bounds around the true effect for the weighted mean exposure mapping. The width of the bounds is increasing in the sensitivity factor. Starting from factor 1.0, our bounds contain the true effect.}
    \label{fig:mean_bounds_over_x}
    \vspace{-0.2cm}
\end{figure}

\textbf{Results:}
\textit{Research question (1): Are our bounds valid?}
$\Rightarrow$ We assess validity by visualizing our bounds for exposure mapping \circled{1} on both the small (Fig.~\ref{fig:mean_bounds_over_x}) and the large dataset (Fig.~\ref{fig:distribution_mean}). We compare our bounds over various specifications of $b^{\pm}$ (= ``factor'' $\times$ true misspecification). We observe that, for a too small sensitivity bound assumption (factor 0.5), the bounds do not completely contain the true effect. For a sufficient bound assumption (factor $\geq$ 1), our bounds are valid.
\begin{figure}[t]
    \centering
    \includegraphics[width=1\linewidth]{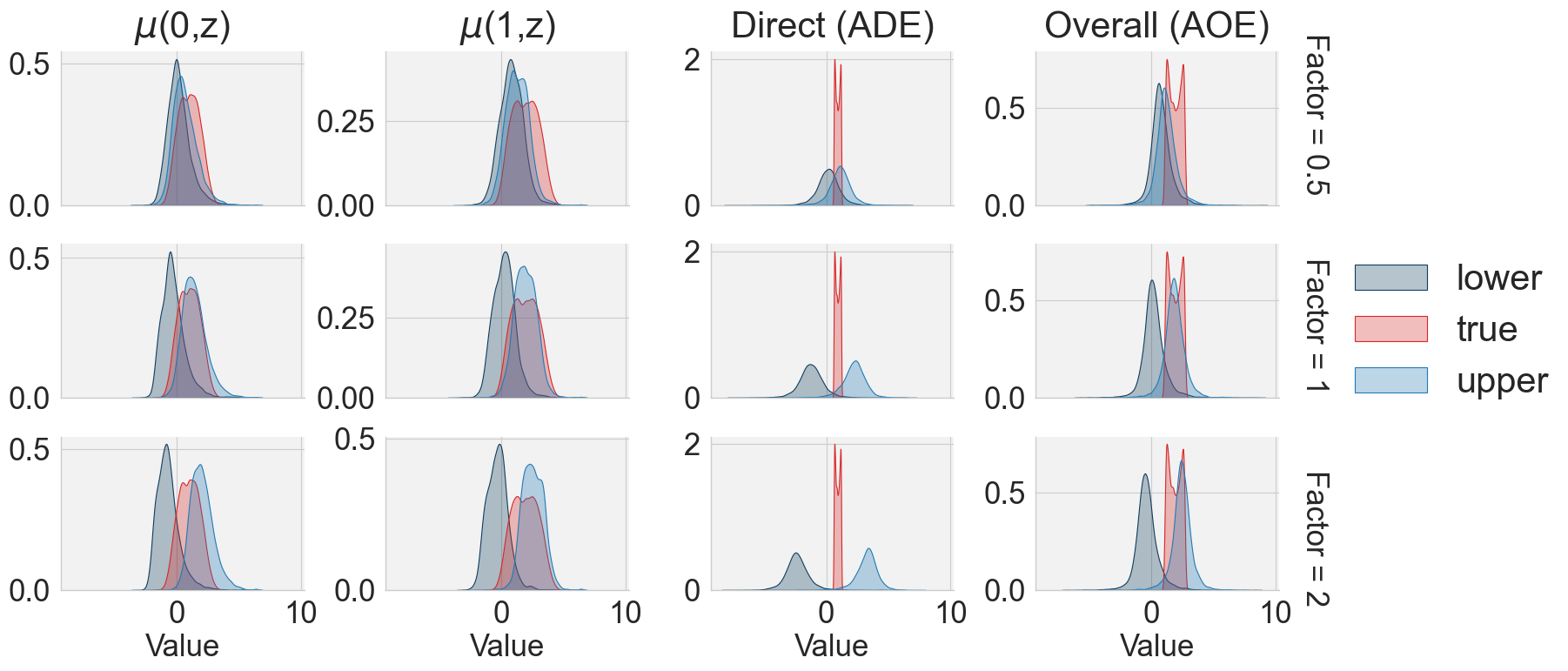}
    \vspace{-0.7cm}
    \caption{\textbf{Distribution of bounds:} Bounds and true potential outcomes and effects for the weighted mean exposure mapping on the large dataset. For sufficiently large sensitivity factor ($\geq$ 1), the distributions of upper and lower bounds enclose the true PO/effect, thus confirming that the bounds are valid.}
    \label{fig:distribution_mean}
    \vspace{-0.6cm}
\end{figure}

\vspace{-0.2cm}
\textit{Research question (2): How does the convergence of our orthogonal estimator compare to a simple plug-in estimator?} $\Rightarrow$ We compare the convergence and the behavior of the coverage of our orthogonal estimator for increasing network size $N$ for setting \circled{2} in Fig.~\ref{fig:ortho_vs_plugin_ade}. As expected: (a) our bounds are \textit{valid even for small sample sizes and are quickly approaching the sharp oracle bounds}, whereas the plug-in bounds fail to provide correct coverage; (b) due to orthogonality, our framework benefits from \emph{faster convergence}.
\vspace{-0.4cm}
\begin{figure}[h]
    \centering
    \begin{subfigure}[b]{0.63\linewidth}
        \hspace{-0.2cm}
        \includegraphics[width=1.1\linewidth]{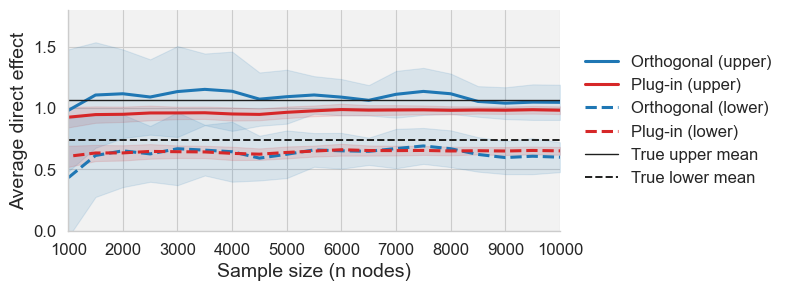}
        \vspace{-0.4cm}
        \caption{}
        \label{fig:ortho_vs_plugin_ade_left}
    \end{subfigure}
    \hfill
    \begin{subfigure}[b]{0.35\linewidth}
        \hspace{-0.2cm}
        \includegraphics[width=1.1\linewidth]{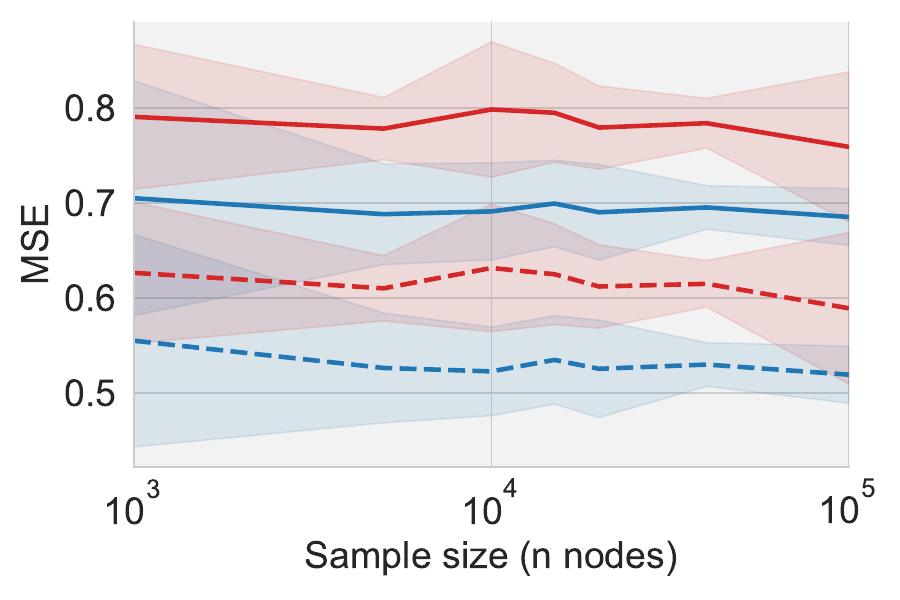}
        \vspace{-0.5cm}
        \caption{}
        \label{fig:ortho_vs_plugin_ade_right}
    \end{subfigure}
    \vspace{-0.2cm}
    \caption{\textbf{Coverage \& convergence:} \emph{(a)} Bounds on the ADE under a threshold exposure mapping increasing number of nodes ($6$-dim covariates). The orthogonal bounds are valid, approaching the sharp oracle bounds, whereas the plug-in bounds are not valid. \emph{(b)} Our orthogonal bounds show faster convergence than the plug-in bounds.}
    \vspace{-0.4cm}
    \label{fig:ortho_vs_plugin_ade}
\end{figure}

\vspace{-0.2cm}
\textit{Research question (3): How informative are our bounds?} $\Rightarrow$ We assess the width of our intervals under exposure mapping \circled{3}. For decision-making, informative bounds (i)~are narrow compared to the outcome range and (ii)~are either strictly positive or negative. Our bounds fulfill both desiderata: (i)~The average width of the ADE intervals over all $z$ with correctly specified sensitivity factors corresponds to merely $8.71\%  \, (\pm 0.37\%)$ of the overall outcome range. (ii)~All of our intervals are strictly bounded away from zero, correctly recognizing the positive treatment effect.


\textbf{Conclusion:}
We proposed a flexible and model-agnostic framework for partial identification of potential outcomes and treatment effects on networks in the presence of exposure mapping misspecification. We derived a robust estimation framework with quasi-oracle rate properties and showed that the estimated bounds remain valid and sharp. Finally, we instantiated our framework with three commonly employed exposure mappings and highlighted the interpretability of our bounds in extensive experiments.

\section*{Acknowledgments}

Miruna Oprescu was supported by the U.S. Department of Energy, Office of Science, Office of Advanced Scientific Computing Research, under Award DE-SC0023112.

\section*{Impact Statement}

This paper presents work whose goal is to advance the field of Machine Learning. There are many potential societal consequences of our work, none which we feel must be specifically highlighted here.


\bibliography{bibliography}
\bibliographystyle{icml2026}

%
%
\newpage
\appendix
\onecolumn

\section{Notation}
\label{sec:appendix_notation}

\bgroup
\def\arraystretch{1.5}
\begin{tabular}{p{1.25in}p{3.25in}}
$\displaystyle \mathcal{G}$ & Network consisting of $N$ nodes\\
$\displaystyle \mathcal{N}, \mathcal{E}$ & Sets of nodes and edges in $\mathcal{G}$\\
$\displaystyle \mathcal{N}_i , \mathcal{N}_{-i}$ & Network partition respective to individual $i$, where $\mathcal{N}_i$ defines the neighborhood of node $i$, i.e., the set of nodes $j$ connected to $i$ by an edge, and $\mathcal{N}_{-i}$ the complement of $\mathcal{N}_{i}$ in $\mathcal{N}$\\
$\displaystyle n_i$ & Degree of node $i$, i.e., number of neighbors of $i$\\
$\displaystyle T_i, T_{\mathcal{N}_{i}}$ & Binary unit and neighborhood treatments\\
$\displaystyle \mathbf{X}$ & Confounders in domain $\mathcal{X}$\\
$\displaystyle Y$ & Outcome in domain $\mathcal{Y}$\\
$\displaystyle g$ & Exposure mapping, $g:[0,1]^n \mapsto \mathcal{Z}$\\
$\displaystyle Z$ & Scalar summarizing the neighborhood exposure, i.e., $Z=g(T_{\mathcal{N}})$\\
$\displaystyle Y(t,z)$ & Potential outcome under unit treatment $t$ and neighborhood exposure $z$\\
$\displaystyle \psi(t,z)$ & Average PO estimator\\
$\displaystyle \mu(t,z,\mathbf{x})$ & CAPO estimator\\
$\displaystyle Q^+(t,z,\mathbf{x}), Q^-(t,z,\mathbf{x})$ & Upper and lower quantiles on the conditional CDF wrt. $b^+(z,\mathbf{x}), b^-(z,\mathbf{x})$\\
$\displaystyle b^+(z,\mathbf{x}), b^-(z,\mathbf{x})$ & Upper and lower bound on the exposure mapping shift due to misspecification\\
$\displaystyle \gamma_{u}^{\pm}(t,z,\mathbf{x})$ & $\mathbb{E}\bigl[ (Y-Q^{\pm})_{+}\mid t,z,\mathbf{x}\bigr]$, where \((u)_+ = \max\{u,0\}\)\\
$\displaystyle \gamma_{l}^{\pm}(t,z,\mathbf{x})$ & $\mathbb{E}\bigl[ (Q^{\pm}-Y)_{+}\mid t,z,\mathbf{x}\bigr]$, where \((u)_+ = \max\{u,0\}\)\\
$\displaystyle \pi^t(\mathbf{x}), \pi^g(z \mid \mathbf{x})$ & Unit node and neighborhood propensity functions\\
$\eta $ & Nuissance functions\\
$\displaystyle \phi^{+}_{t,z}(S,\eta), \phi^{-}_{t,z}(S,\eta)$ & Upper and lower orthogonal pseudo-outcome\\
$\displaystyle \mu^{\pm}_\mathrm{DR}(t,z,\mathbf{x})$ & Orthogonal upper and lower CAPO bound estimator\\
$\displaystyle \psi^{+}_\mathrm{DR}(t,z), \psi^{-}_{DR}(t,z)$ & Orthogonal upper and lower average PO bound estimator\\
\end{tabular}
\egroup

\newpage
\section{Extended related work}
\label{sec:appendix_related_work}

Below, we discuss related work on (i)~network interference (Appendix~\ref{sec:appendix_rw_interference}) and (ii)~partial identification methods (Appendix~\ref{sec:appendix_rw_partial_id}). In Appendix~\ref{sec:appendix_rw_interference}, we first provide an overview of the related but non-discussed fields on graph neural networks (GNNs) for interference modeling and causal methods for spatial interference. Then, we give a detailed overview of other works addressing misspecified exposure mappings and highlight how our work differentiates itself from these works. In Appendix~\ref{sec:appendix_rw_partial_id}, we give a brief overview of sensitivity methods for partial identification.

\subsection{Network interference} \label{sec:appendix_rw_interference}

\textbf{GNNs:}
Standard graph ML models fail to estimate causal effects on networks as they follow a different optimization goal \cite{Jiang.2022}. Furthermore, these methods are computationally inefficient, thereby rendering the application to large network data challenging or impossible \cite{Lin.2025}. Therefore, multiple methods for learning a neighborhood representation of the covariates through a GNN have been proposed \citep[e.g.,][]{Adhikari.2025, Jiang.2022, Lin.2025, Ma.2021}. However, these methods commonly assume a known exposure mapping of the neighborhood treatments.

\textbf{Spatial interference:}
In environmental science, treatment effect estimation often faces the challenge of spatial interference, i.e., treatments from different locations affecting the outcomes at other locations. Here, data are often assumed to stem from a spatial grid, meaning that distances between nodes (=spatial cells) and the number of neighboring nodes are fixed for the entire network. Spatial causal inference approaches commonly assume a correctly specified exposure mapping as in approaches targeting network interference \citep[e.g.,][]{Anselin.1988, Giffin.2023, Hanks.2015, Papadogeorgou.2023}. More recent approaches \citep[e.g.,][]{Ali.2024, Khot.2025, Oprescu.2025} assume \emph{localized interference} based on a specified neighborhood radius and employ deep learning methods to capture the latent interference structure. Overall, all methods rely on various types of \textit{correctly specified} exposure mappings.

\textbf{Misspecified exposure mappings:}
Only very few works consider causal effect estimation under a misspecified exposure mapping or network uncertainty.  Most works target estimation on unknown networks, i.e., when there is uncertainty about the existence of certain edges in the network. \citet{Egami.2021} provides bounds on average causal effects under network misspecification in RCTs. \citet{Savje.2021} further shows that, under unknown but limited interference in RCTs, average effects can be identified by certain standard estimators. In a follow-up work, \citet{Savje.2024} assesses the bias of treatment effect estimators when there is a mismatch between the exposure mappings at experiment and inference time. \citet{Ohnishi.2025} learn the latent structure of interference under a Bayesian prior to estimate causal effects under an arbitrary, unknown interference structure. However, the method is only applicable to randomized control trials (RCTs) and targets specific sub-effects different from the standard CATE and ATE.

In the more general setting of observational data, \citet{Weinstein.2023} derive bounds on the bias arising from estimating causal effects under a misspecified network. In follow-up work, \citet{Weinstein.2025} and \citet{Zhang.2025} propose frameworks for estimating causal effects when only proxy networks are available. Similarly, \citet{Zhang.2023} models uncertain interaction using linear graphical causal models, quantifies bias when iid (SUTVA) is incorrectly assumed, and presents a procedure to remove such bias and derive bounds for \textit{average} causal effects.

Other works more similar to our work focus on uncertainty in the neighborhood radius.  \citet{Leung.2022} considers approximate neighborhood interference, allowing treatments assigned to units further from the unit of interest to have potentially nonzero, but smaller, effects on the unit's outcome. In contrast to our work, the proposed method is restricted to the specific type of misspecification and only targets the average overall effect. 
\citet{Belloni.2022} consider estimation under an unknown neighborhood radius, similar to our third use-case. However, the proposed method needs strong modeling assumptions and only applies to the \textit{average} direct effect.

Orthogonal to our work, \citet{Hoshino.2024} propose an \emph{instrumental exposure mapping} to summarize the spillover effects into a low-dimensional variable in instrumental variable regression settings. They show that the resulting estimands for \textit{average} effects are interpretable even if the neighborhood radius is misspecified. 

Overall, there does \textbf{not} exist a general framework for bounding potential outcomes and treatment effects under various types of exposure mapping misspecification for both experimental and observational data. This is our contribution.

\subsection{Partial identification}
\label{sec:appendix_rw_partial_id}

\textbf{Sensitivity analysis as partial identification:}
A commonly employed tool for partial identification is causal sensitivity analysis (CSA). Instead of point-identifying an estimand under the strong assumptions of \emph{no unobserved confounding}, CSA allows unobserved confounding up to a specified confounding strength and derives bounds for causal quantities. A broad range of sensitivity models has been proposed, differing in what aspect of the data-generating process is perturbed and how deviations are parameterized \citep[e.g.,][]{Rosenbaum.1983, Robins.2000, Vansteelandt.2006}. 

\textbf{Marginal sensitivity model (MSM) and extensions:}
Much of the recent literature centers on the MSM \cite{Tan.2006b}, where bounds are obtained by optimizing over admissible propensity reweightings. Recent works show that na{\"i}ve procedures can be conservative and derive \emph{sharp} bound characterizations and estimators \cite{Dorn.2022,Dorn.2025}, which also enables efficient learning of \emph{CATE bounds} via meta-learning \cite{Oprescu.2023}. Beyond binary treatments and standard treatment effect queries, other works propose continuous-treatment marginal sensitivity models \cite{Jesson.2022}, generalized sensitivity models with sharp bounds for broader causal queries \cite{Frauen.2023b}, and neural frameworks that automate generalized sensitivity analysis across model classes and treatment types \cite{Frauen.2024}. Sensitivity-style partial identification has also been used in adjacent ML problems such as confounding-robust policy learning \cite{Hess.2026, Kallus.2018, Kallus.2019}, partial identification of counterfactual queries \cite{Melnychuk.2023b}, survival analysis~\cite{Wang.2025b}, sensitivity auditing of causal fairness \cite{Kilbertus.2019, Schroder.2024b}, and modern uncertainty quantification, e.g., conformal-style intervals for ITEs at a given sensitivity level \cite{Yin.2022b}.

\newpage
\section{Extended theory}
\label{sec:appendix_theory}

\subsection{Summary of bounds}

\begin{table}[h]
    \centering
    \begin{tabular}{c|l}
    \toprule
        \multirow{2}{*}{Potential outcomes} & $\mu^{+}(t,z, \mathbf{x}) = Q^{+}(t,z,\mathbf{x})+ \frac{1}{b^{-}(z,\mathbf{x})}\gamma_u^+(t,z,\mathbf{X}) - \frac{1}{b^{+}(z,\mathbf{x})}\gamma_u^-(t,z,\mathbf{X})$\\
        & $\mu^{-}(t,z, \mathbf{x}) = Q^{-}(t,z,\mathbf{x})+ \frac{1}{b^{+}(z,\mathbf{x})}\gamma_u^-(t,z,\mathbf{X})- \frac{1}{b^{-}(z,\mathbf{x})}\gamma_l^-(t,z,\mathbf{X})$\\
        \midrule
         & $\phi^+_{t,z}(S;\widehat\eta) = \widehat Q^+(t,z,\mathbf{X}) +\frac{\widehat\gamma_u^+(t,z,\mathbf{X})}{b^{-}(z,\mathbf{X})}
            -\frac{\widehat\gamma_l^+(t,z,\mathbf{X}) }{b^{+}(z,\mathbf{X})}$ \\
        Pseudo-outcomes & $\qquad \qquad \quad +\frac{\mathbf{1}_{[T=t]}\,\mathbf{1}_{[Z=z]}}
            {\widehat\pi^t(\mathbf{X})\,\widehat\pi^g(Z\mid\mathbf{X})}
            \Bigg[\frac{(Y-\widehat Q^+(t,Z,\mathbf{X}))_{+}-\widehat\gamma_u^+(t,Z,\mathbf{X})}{b^{-}(Z,\mathbf{X})}-\frac{(\widehat Q^+(t,Z,\mathbf{X})-Y)_{+}-\widehat\gamma_l^+(t,Z,\mathbf{X})}{b^{+}(Z,\mathbf{X})}\Bigg]$\\
        (discrete) &
            $\phi^-_{t,z}(S;\widehat\eta) = \widehat Q^-(t,z,\mathbf{X}) +\frac{\widehat\gamma_u^-(t,z,\mathbf{X})}{b^{+}(z,\mathbf{X})}
            -\frac{\widehat\gamma_l^-(t,z,\mathbf{X}) }{b^{-}(z,\mathbf{X})}$\\
            & $\qquad \qquad \quad +\frac{\mathbf{1}_{[T=t]}\,\mathbf{1}_{[Z=z]}}
            {\widehat\pi^t(\mathbf{X})\,\widehat\pi^g(Z\mid\mathbf{X})}
            \Bigg[\frac{(Y-\widehat Q^-(t,Z,\mathbf{X}))_{+}-\widehat\gamma_u^-(t,Z,\mathbf{X})}{b^{+}(Z,\mathbf{X})}-\frac{(\widehat Q^-(t,Z,\mathbf{X})-Y)_{+}-\widehat\gamma_l^-(t,Z,\mathbf{X})}{b^{-}(Z,\mathbf{X})}\Bigg]$\\
        \midrule
         & $\phi^+_{t,z}(S;\widehat\eta) = \widehat Q^+(t,z,\mathbf{X}) +\frac{\widehat\gamma_u^+(t,z,\mathbf{X})}{b^{-}(z,\mathbf{X})}
            -\frac{\widehat\gamma_l^+(t,z,\mathbf{X}) }{b^{+}(z,\mathbf{X})}$ \\
        Pseudo-outcomes & $\qquad \qquad \quad +\frac{\mathbf{1}_{[T=t]}\,K_h(Z-z)}
            {\widehat\pi^t(\mathbf{X})\,\widehat\pi^g(Z\mid\mathbf{X})}
            \Bigg[\frac{(Y-\widehat Q^+(t,Z,\mathbf{X}))_{+}-\widehat\gamma_u^+(t,Z,\mathbf{X})}{b^{-}(Z,\mathbf{X})}-\frac{(\widehat Q^+(t,Z,\mathbf{X})-Y)_{+}-\widehat\gamma_l^+(t,Z,\mathbf{X})}{b^{+}(Z,\mathbf{X})}\Bigg]$\\
        (continuous) &
            $\phi^-_{t,z}(S;\widehat\eta) = \widehat Q^-(t,z,\mathbf{X}) +\frac{\widehat\gamma_u^-(t,z,\mathbf{X})}{b^{+}(z,\mathbf{X})}
            -\frac{\widehat\gamma_l^-(t,z,\mathbf{X}) }{b^{-}(z,\mathbf{X})}$\\
            & $\qquad \qquad \quad +\frac{\mathbf{1}_{[T=t]}\,K_h(Z-z)}
            {\widehat\pi^t(\mathbf{X})\,\widehat\pi^g(Z\mid\mathbf{X})}
            \Bigg[\frac{(Y-\widehat Q^-(t,Z,\mathbf{X}))_{+}-\widehat\gamma_u^-(t,Z,\mathbf{X})}{b^{+}(Z,\mathbf{X})}-\frac{(\widehat Q^-(t,Z,\mathbf{X})-Y)_{+}-\widehat\gamma_l^-(t,Z,\mathbf{X})}{b^{-}(Z,\mathbf{X})}\Bigg]$\\
        \bottomrule
    \end{tabular}
    \caption{Summary of our bounds from the main paper.}
    \label{tab:bounds_overview}
\end{table}

\subsection{Orthogonal lower bound}

In Section~\ref{sec:method}, we provided an orthogonal estimation framework for the upper bound of the potential outcomes and treatment effects. For completeness, we now also provide the formulation for the lower bounds:

Let $S=(\mathbf{X},Y,T,Z)$. Fix $(t,z)$. Define the localization weight
\begin{equation*}
\omega_{z,h}(Z):=
\begin{cases}
\mathbf{1}_{[Z=z]}, & \text{if } Z\ \text{binary/discrete},\\
K_h(Z-z), & \text{if } Z\ \text{continuous},
\end{cases}
\end{equation*}
and let $\pi^g(Z\mid\mathbf{X})$ denote the conditional probability mass function (discrete $Z$) or density (continuous $Z$). Let $\widehat\eta=(\widehat\pi^t,\widehat\pi^g,\widehat Q^-,\widehat\gamma_u^-,\widehat\gamma_l^-)$ be a set of estimated nuisances. Then, an orthogonal pseudo-outcome for the CAPO lower bound $\mu^-(t,z,\mathbf{x})$ is:
\begin{equation}
\begin{aligned}
\phi^-_{t,z}(S;\widehat\eta) = &\widehat Q^-(t,z,\mathbf{X}) +\frac{\widehat\gamma_u^-(t,z,\mathbf{X})}{b^{+}(z,\mathbf{X})}
-\frac{\widehat\gamma_l^-(t,z,\mathbf{X}) }{b^{-}(z,\mathbf{X})}\\
& +\frac{\mathbf{1}_{[T=t]}\,\omega_{z,h}(Z)}
{\widehat\pi^t(\mathbf{X})\,\widehat\pi^g(Z\mid\mathbf{X})}
\Bigg[
\frac{(Y-\widehat Q^-(t,Z,\mathbf{X}))_{+}-\widehat\gamma_u^-(t,Z,\mathbf{X})}{b^{+}(Z,\mathbf{X})}
-\frac{(\widehat Q^-(t,Z,\mathbf{X})-Y)_{+}-\widehat\gamma_l^-(t,Z,\mathbf{X})}{b^{-}(Z,\mathbf{X})}
\Bigg],
\end{aligned}
\end{equation}
where $\gamma^{-}_u(t,z,\mathbf{x}):= \mathbb{E}\!\left[(Y-Q^{-}(\cdot))_{+}\mid t,z,\mathbf{x}\right]$ and 
$\gamma^{-}_l(t,z,\mathbf{x}):= \mathbb{E}\!\left[(Q^{-}(\cdot)-Y)_{+}\mid t,z,\mathbf{x}\right]$.

\subsection{Bounds on the treatment effects}

Recall the definition of the average and individual direct effects 
$$\tau_d^{(t,z),(t^{'},z)} := \psi(t,z) - \psi(t^{'},z) \text{ (ADE)} \quad \text{and} \quad \tau_{d_i}^{(t,z),(t^{'},z)}(\mathbf{x}_i) := \mu(t,z,\mathbf{x}_i) - \mu(t^{'},z, \mathbf{x}_i) \text{ (IDE)},$$
spillover/indirect effects
$$\tau_s^{(t,z),(t,z^{'})} := \psi(t,z) - \psi(t,z^{'}) \text{ (ASE)} \quad \text{and} \quad \tau_{s_i}^{(t,z),(t,z^{'})}(\mathbf{x}_i) :=  \mu(t,z,\mathbf{x}_i) - \mu(t,z^{'}, \mathbf{x}_i) \text{ (ISE)},$$
and overall effects
$$ \tau_o^{(t,z),(t^{'},z^{'})} := \psi(t,z) - \psi(t^{'},z^{'})
\text{ (AOE)} \quad \text{and} \quad \tau_{o_i}^{(t,z),(t^{'},z^{'})}(\mathbf{x}_i) := \mu(t,z,\mathbf{x}_i) - \mu(t^{'},z^{'}, \mathbf{x}_i)\text{ (IOE)}.$$
Based on the CAPO bounds $\mu^{\pm}(t,z,\mathbf{x})$ from Theorem~\ref{thm:plug_in_bounds}, we thus obtain the treatment effects through the general formula $\tau^{+ \; (a,b)} = f^{+}(a,\cdot) - f^{-}(b,\cdot)$ and $\tau^{- \; (a,b)} = f^{-}(a,\cdot) - f^{+}(b,\cdot)$, where with a slight abuse of notation $\tau$ refers to any of the (conditional) effects above, $f$ refers to either $\mu$ or $\psi$, and $(a,b)$ denotes the change in $t$ and/or $z$. Specifically, the conditional treatment effects IDE / ISE / IOE are identified as follows.

\textbf{Direct effect:}
\begin{align}
    \tau_{d_{i}}^{+ \;(t,z),(t^{'},z)}(\mathbf{x}) =  
         &Q^{+}(t,z,\mathbf{x})-Q^{-}(t^{'},z,\mathbf{x}) + \frac{1}{b^{-}(z,\mathbf{x})}(\gamma_u^{+}(t,z,\mathbf{x})-\gamma_l^{-}(t^{'},z,\mathbf{x}))\\ &- \frac{1}{b^{+}(z,\mathbf{x})}(\gamma_l^{+}(t,z,\mathbf{x})-\gamma_u^{-}(t^{'},z,\mathbf{x}))\\
        \tau_{d_{i}}^{- \; (t,z),(t^{'},z)}(\mathbf{x}) = &Q^{-}(t,z,\mathbf{x})-Q^{+}(t^{'},z,\mathbf{x})+ \frac{1}{b^{+}(z,\mathbf{x})}(\gamma_u^{-}(t,z,\mathbf{x})-\gamma_l^{+}(t^{'},z,\mathbf{x})) \\ &- \frac{1}{b^{-}(z,\mathbf{x})}(\gamma_l^{-}(t,z,\mathbf{x})-\gamma_u^{+}(t^{'},z,\mathbf{x}))
\end{align}

\textbf{Indirect/spillover effect:}
\begin{align}
    \tau_{d_{s}}^{+ \;(t,z),(t,z^{'})}(\mathbf{x}) =  
         &Q^{+}(t,z,\mathbf{x})-Q^{-}(t,z^{'},\mathbf{x})+ \frac{1}{b^{-}(z,\mathbf{x})}\gamma_u^{+}(t,z,\mathbf{x}) \\ &+ \frac{1}{b^{+}(z^{'},\mathbf{x})}\gamma_u^{-}(t,z^{'},\mathbf{x})- \frac{1}{b^{-}(z^{'},\mathbf{x})}\gamma_l^{-}(t,z^{'},\mathbf{x}) - \frac{1}{b^{+}(z,\mathbf{x})}\gamma_l^{+}(t,z,\mathbf{x})\\
        \tau_{d_{s}}^{- \; (t,z),(t,z^{'})}(\mathbf{x}) = &Q^{-}(t,z,\mathbf{x})-Q^{+}(t,z^{'},\mathbf{x})+ \frac{1}{b^{+}(z,\mathbf{x})}\gamma_u^{-}(t,z,\mathbf{x})\\ &+\frac{1}{b^{-}(z^{'},\mathbf{x})}\gamma_u^{+}(t,z^{'},\mathbf{x})- \frac{1}{b^{+}(z^{'},\mathbf{x})}\gamma_l^{+}(t,z^{'},\mathbf{x}) - \frac{1}{b^{-}(z,\mathbf{x})}\gamma_l^{-}(t,z,\mathbf{x})
\end{align}

\textbf{Overall effect:}
\begin{align}
    \tau_{d_{o}}^{+ \;(t,z),(t^{'},z^{'})}(\mathbf{x}) =  
         &Q^{+}(t,z,\mathbf{x})-Q^{-}(t^{'},z^{'},\mathbf{x})+ \frac{1}{b^{-}(z,\mathbf{x})}\gamma_u^{+}(t,z,\mathbf{x}) \\ &+ \frac{1}{b^{+}(z^{'},\mathbf{x})}\gamma_u^{-}(t^{'},z^{'},\mathbf{x})- \frac{1}{b^{-}(z^{'},\mathbf{x})}\gamma_l^{-}(t^{'},z^{'},\mathbf{x}) - \frac{1}{b^{+}(z,\mathbf{x})}\gamma_l^{+}(t,z,\mathbf{x})\\
        \tau_{d_{o}}^{- \; (t,z),(t^{'},z^{'})}(\mathbf{x}) = &Q^{-}(t,z,\mathbf{x})-Q^{+}(t^{'},z^{'},\mathbf{x})+ \frac{1}{b^{+}(z,\mathbf{x})}\gamma_u^{-}(t,z,\mathbf{x})\\ &+\frac{1}{b^{-}(z^{'},\mathbf{x})}\gamma_u^{+}(t^{'},z^{'},\mathbf{x})- \frac{1}{b^{+}(z^{'},\mathbf{x})}\gamma_l^{+}(t^{'},z^{'},\mathbf{x}) - \frac{1}{b^{-}(z,\mathbf{x})}\gamma_l^{-}(t,z,\mathbf{x})
\end{align}

The bounds on the average effects ADE, AIE, and AOE are then identified by the expectation of the individual effects over the covariates $\mathbf{X}$.

\subsection{Continuous neighborhood exposure}
\label{sec:appendix_continuousZ}

This section gives the continuous-$Z$ analogues of Theorem~\ref{thm:second_order_remainder} and Corollary~\ref{cor:quasi_oracle_and_inference}, as well as the corresponding \emph{sharpness} and \emph{validity} guarantees for the estimated bounds (complementary to the the discrete-$Z$ results in the main text).

When $Z$ is continuous, point evaluation at $Z=z$ is non-regular. Following the standard
approach in orthogonal learning for continuous exposures, we therefore target a
\emph{kernel-localized} (bandwidth-indexed) version of the bound functional. Under smoothness
in $z$, these localized targets converge to the original (pointwise) bounds as $h\downarrow 0$,
at the usual bias--variance tradeoff governed by $(n,h)$.

\begin{assumption}[Kernel localization]
\label{assump:kernel}
Let $Z$ be continuous and let $K_h(u)=\frac{1}{h}K(u/h)$, where $K$ is bounded, integrates to $1$,
and $\int K(u)^2\,du<\infty$. Let $h=h_n\downarrow 0$ with $nh_n\to\infty$.
\end{assumption}

\paragraph{Kernel-localized targets.}
Fix $(t,z)$ and let $h>0$. For continuous $Z$, define the localized selection weight
\begin{equation}
\label{eq:kappa_tzh}
\kappa_{t,z,h}(S)
:=
\frac{\mathbf{1}_{[T=t]}\,K_h(Z-z)}{\pi^t(\mathbf{X})\,\pi^g(Z\mid \mathbf{X})},
\end{equation}
as in the continuous-$Z$ modification of the proof of Theorem~\ref{thm:orthogonal_bounds}.
Define the kernel-localized pseudo-outcome $\phi^+_{t,z,h}(S;\widehat\eta)$ as
Eq.~\eqref{eq:pseudo_outcome_unified} with $\omega_{z,h}(Z)=K_h(Z-z)$.

When $\widehat\eta=\eta$, define the associated bandwidth-indexed functionals by
\begin{equation}
\label{eq:mu_psi_h_defs}
\mu_h^+(t,z,\mathbf{x})
:=
\mathbb{E}\!\left[\phi^+_{t,z,h}(S;\eta)\mid \mathbf{X}=\mathbf{x}\right],
\qquad
\psi_h^+(t,z)
:=
\mathbb{E}\!\left[\phi^+_{t,z,h}(S;\eta)\right].
\end{equation}
Under standard smoothness in $z$, $\mu_h^+(t,z,\mathbf{x})\to \mu^+(t,z,\mathbf{x})$ and
$\psi_h^+(t,z)\to \psi^+(t,z)$ as $h\downarrow 0$
(see Remark~\ref{rem:unbiasedness_phi}).

\medskip

\noindent
Relative to the discrete-$Z$ case, kernel localization inflates the second-order remainder by
a factor $h^{-1/2}$ (reflecting $\int K_h^2 = O(1/h)$). This propagates to the final-stage CAPO rate
and yields the usual $\sqrt{nh}$ scaling for the (smoothed) APO.

\begin{restatable}[Second-order nuisance error (continuous $Z$)]{theorem}{nuisanceReminderCont}
\label{thm:second_order_remainder_continuous}
Assume $Z$ is continuous and Assumptions~\ref{assump:overlap_bounded} and~\ref{assump:kernel} hold.
Let $\widehat\eta=(\widehat\pi^t,\widehat\pi^g,\widehat Q^+,\widehat\gamma_u^+,\widehat\gamma_l^+)$ be the
cross-fitted nuisances used in $\phi^+_{t,z,h}(S;\widehat\eta)$ (Eq.~\eqref{eq:pseudo_outcome_unified} with
$\omega_{z,h}(Z)=K_h(Z-z)$).

Define nuisance error rates (in $L_2$ norms over the appropriate arguments) by
\begin{align}
r_{n,\pi}:=\|\widehat\pi^t-\pi^t\|_2+\|\widehat\pi^g-\pi^g\|_2,
\qquad
r_{n,Q}:=\|\widehat Q^+-Q^+\|_2,
\end{align}
\begin{align}
r_{n,\gamma}:=\|\widehat\gamma_u^+-\gamma_u(\widehat Q^+;\cdot)\|_2
+\|\widehat\gamma_l^+-\gamma_l(\widehat Q^+;\cdot)\|_2,
\end{align}
where the norms are taken over the random variables that the corresponding nuisance is evaluated on
(e.g., $(Z,\mathbf{X})$ for $\pi^g(Z\mid\mathbf{X})$, $Q^+(t,Z,\mathbf{X})$, and $\gamma^\pm(t,Z,\mathbf{X})$).

Then, the conditional bias induced by nuisance estimation satisfies
\begin{align}
\left\|\mathbb{E}\!\left[\phi^+_{t,z,h}(S;\widehat\eta)-\phi^+_{t,z,h}(S;\eta)\mid \mathbf{X}\right]\right\|_2
=
O_p\!\left(\frac{r_{n,\pi}\,r_{n,\gamma}+r_{n,Q}^2}{\sqrt{h}}\right).
\end{align}
\end{restatable}

\begin{restatable}[Quasi-oracle rates and inference (continuous $Z$)]{corollary}{quasiOracleRatesCont}
\label{cor:quasi_oracle_and_inference_continuous}
Assume the conditions of Theorem~\ref{thm:second_order_remainder_continuous} and that the
second-stage regression learner $\widehat{\mathbb{E}}_n[\cdot\mid\mathbf{X}=\mathbf{x}]$ satisfies
Assumption~\ref{assump:second_stage} with rate $\delta_n$ when regressing
$\phi^+_{t,z,h}(S;\eta)$ on $\mathbf{X}$.

Then:

\underline{CAPO rates:} The CAPO upper-bound estimator satisfies
\begin{align}
\|\widehat\mu_h^+(t,z,\cdot)-\mu_h^+(t,z,\cdot)\|_2
=
O_p\!\left(\delta_n+\frac{r_{n,\pi}\,r_{n,\gamma}+r_{n,Q}^2}{\sqrt{h}}\right).
\end{align}

\underline{APO rates:} The APO upper-bound estimator $\widehat\psi^+(t,z)=\mathbb{E}_n[\widehat\phi^+_{t,z, h}]$ satisfies
\begin{align}
|\widehat\psi_h^+(t,z)-\psi_h^+(t,z)|
=
O_p\!\left(\frac{1}{\sqrt{n h}}
+\frac{r_{n,\pi}\,r_{n,\gamma}+r_{n,Q}^2}{\sqrt{h}}\right).
\end{align}

\underline{$\sqrt{nh}$-CLT (central limit theorem) for the (smoothed) APO.}
If $r_{n,\pi}\,r_{n,\gamma}+r_{n,Q}^2=o_p(n^{-1/2})$, then
\begin{align}
\sqrt{n h}\left(\widehat\psi_h^+(t,z)-\psi_h^+(t,z)\right)
\ \rightsquigarrow\
\mathcal{N}\!\left(0,\;V_h^+(t,z)\right),
\end{align}
where one valid asymptotic variance target is
$V_h^+(t,z):=\mathrm{Var}\!\big(\sqrt{h}\,\phi^+_{t,z,h}(S;\eta)\big)$.

Finally, if the smoothing bias satisfies $|\psi_h^+(t,z)-\psi^+(t,z)|=o((nh)^{-1/2})$ (e.g., via undersmoothing under $z$-smoothness), then the same CLT holds with $\psi^+(t,z)$ in place of $\psi_h^+(t,z)$.
\end{restatable}

\paragraph{Sharpness and validity of the estimated bounds.}
The previous results control the second-order remainder and deliver quasi-oracle rates for the localized targets $(\mu_h^+,\psi_h^+)$. We now record the two complementary guarantees from the main text in their continuous-$Z$ versions: (i)~consistency for the \emph{sharp} identified bounds, and (ii)~\emph{validity} of the resulting intervals under potentially misspecified cutoffs. As before, the statements hold for both endpoints $(+/-)$; we write them for the upper endpoint for brevity, with the lower endpoint following analogously by sign-swapping in the pseudo-outcome.

\begin{restatable}[Consistency for sharp bounds (continuous $Z$)]{proposition}{sharpnessGuaranteeCont}
\label{prop:sharpness_continuous}
Assume the conditions of Corollary~\ref{cor:quasi_oracle_and_inference_continuous} and consider the
corresponding lower-bound estimator $\widehat\mu_h^-(t,z,\cdot)$ constructed from the lower-bound
pseudo-outcome (defined analogously to Eq.~\eqref{eq:pseudo_outcome_unified}). Suppose $\delta_n=o_p(1)$ and
\begin{align}
\frac{r_{n,\pi}\,r_{n,\gamma}+r_{n,Q}^2}{\sqrt{h}}=o_p(1).
\end{align}
Then,
\begin{align}
\|\widehat\mu_h^\pm(t,z,\cdot)-\mu_h^\pm(t,z,\cdot)\|_2=o_p(1),
\qquad
|\widehat\psi_h^\pm(t,z)-\psi_h^\pm(t,z)|=o_p(1).
\end{align}
Consequently, the estimated CAPO and APO intervals converge to the \emph{sharp} kernel-localized identified intervals for the bandwidth-indexed targets.

Moreover, if the smoothing bias vanishes at the appropriate rate (e.g., $|\psi_h^\pm(t,z)-\psi^\pm(t,z)|=o((nh)^{-1/2})$), then the estimated intervals are asymptotically sharp for the original pointwise bounds as $h\downarrow 0$.
\end{restatable}

\begin{restatable}[Asymptotic validity under misspecified cutoffs (continuous $Z$)]{corollary}{validityCutoffsCont}
\label{cor:validity_misspecified_cutoffs_continuous}
Fix measurable cutoffs $\overline{Q}^\pm(t,z,\mathbf{x})$ (not necessarily equal to the sharp cut-offs) and let $\overline\mu_h^\pm(t,z,\mathbf{x};\overline{Q}^\pm)$ and $\overline\psi_h^\pm(t,z;\overline{Q}^\pm)$ denote the resulting (possibly non-sharp) kernel-localized bound functionals induced by these cutoffs (i.e., the targets obtained by replacing $Q^\pm$ in the pseudo-outcomes and taking the conditional/unconditional expectations as in Eq.~\eqref{eq:mu_psi_h_defs}). Then, the induced intervals
\begin{align}
\big[\overline\mu_h^-(t,z,\mathbf{x};\overline Q^-),\ \overline\mu_h^+(t,z,\mathbf{x};\overline Q^+)\big]
\quad\text{and}\quad
\big[\overline\psi_h^-(t,z;\overline Q^-),\ \overline\psi_h^+(t,z;\overline Q^+)\big]
\end{align}
are (not necessarily sharp) \underline{valid} CAPO and APO intervals for the kernel-localized targets.

Moreover, if $\widehat Q^\pm \to \overline Q^\pm$ in $L_2$ and either
\begin{itemize}
\item[(i)] $(\widehat\pi^t,\widehat\pi^g)$ is consistent, \emph{or}
\item[(ii)] the corresponding tail-moment regressions $(\widehat\gamma_u^\pm,\widehat\gamma_l^\pm)$ are consistent for the
targets induced by $\overline Q^\pm$,
\end{itemize}
then the estimated endpoints converge to the induced (conservative) targets and the resulting (C)APO intervals remain asymptotically valid, though potentially conservative. If $\overline Q^\pm$ equals the sharp cut-offs, then the induced bounds coincide with the sharp bounds, and the intervals are asymptotically sharp as well.
\end{restatable}

\paragraph{Conclusion.}
For continuous neighborhood exposure, our estimation and theory proceed exactly as in the discrete-$Z$ case, except that (i)~the indicator $\mathbf{1}_{[Z=z]}$ in the selection weight is replaced by kernel localization $K_h(Z-z)$ and (ii)~the conditional pmf $\pi^g(z\mid \mathbf{X})$ is replaced by the conditional density $\pi^g(Z\mid \mathbf{X})$. This replacement yields an effective sample size $nh$ around $z$, which inflates the second-order remainder by a factor $h^{-1/2}$ and leads to $\sqrt{nh}$ scaling for APO inference. Under smoothness in $z$, the bandwidth-indexed targets $(\mu_h^\pm,\psi_h^\pm)$ converge to the point-wise bounds $(\mu^\pm,\psi^\pm)$ as $h\downarrow 0$, yielding the usual bias-variance tradeoff in $(n,h)$. The proofs in Supplement~\ref{sec:appendix_proofs} show that all continuous-$Z$ results follow from the discrete-$Z$ proofs by replacing $\mathbf{1}_{[Z=z]}$ by $K_h(Z-z)$ and tracking $\int K_h^2 = O(1/h)$.

\newpage
\section{Proofs}
\label{sec:appendix_proofs}

\subsection{Justification of the setting-specific $b^+, b^-$ }

Below we give a justification for the specification of $b^+, b^-$ for exposure mappings \circled{1} and \circled{2}.

\circled{1} \textbf{Weighted mean exposure:}
Define $g(t_{\mathcal{N}}) := \sum_{j \in \mathcal{N}}\frac{t_j}{n} = \frac{N_T}{n}$, where $N_T$ denotes the number of treated neighbors and $n$ denotes the neighborhood size. We assume $g^{\ast}(t_{\mathcal{N}}) = \sum_{j \in \mathcal{N}}w_jt_j$, where $|\frac{1}{n}-w_j| \geq \varepsilon$ for all $j$. 

First observe that 
\begin{align}
    b^{-}(z,\mathbf{x}) \leq \frac{P(\sum_{j \in \mathcal{N}}w_jt_j = z \mid \mathbf{x})}{P(\sum_{j \in \mathcal{N}}\frac{t_j}{n}=z \mid \mathbf{x})} \leq b^{+}(z,\mathbf{x}) \iff b^{-}(z,\mathbf{x}) \leq \frac{G(z\mid \mathbf{x})-G(s\mid \mathbf{x})}{F(z\mid \mathbf{x})-F(s\mid \mathbf{x})} \leq b^{+}(z,\mathbf{x})
\end{align}
for all $s \in \mathcal{Z}$, where $G(\cdot)$ and $F(\cdot)$ denote the conditional cumulative distribution function of $g^{\ast}(T_{\mathcal{N}})$ and $g(T_{\mathcal{N}})$. Since $|\frac{1}{n}-w_j| \geq \varepsilon$, it holds that, for all $k \in \mathcal{Z}$, we have
\begin{align}
    &P\left((\frac{1}{n}+\varepsilon)\sum_{j \in \mathcal{N}}t_j \leq k \mid \mathbf{x}\right) \leq P\left(\sum_{j \in \mathcal{N}}w_jt_j \leq k \mid \mathbf{x}\right) \leq P\left((\frac{1}{n}-\varepsilon)\sum_{j \in \mathcal{N}}t_j \leq k \mid \mathbf{x}\right)\\
    \iff
    &P\left( \sum_{j \in \mathcal{N}}\frac{t_j}{n} \leq \frac{k}{1+n\varepsilon} \mid \mathbf{x}\right) \leq P\left(\sum_{j \in \mathcal{N}}w_jt_j \leq k \mid \mathbf{x}\right) \leq P\left(\sum_{j \in \mathcal{N}}\frac{t_j}{n} \leq \frac{k}{1-n\varepsilon} \mid \mathbf{x}\right) .
\end{align}

Therefore, we can bound the enumerator $G(z\mid \mathbf{x}) - G(s\mid \mathbf{x})$ by
\footnotesize
\begin{align}
    P\left( \sum_{j \in \mathcal{N}}\frac{t_j}{n} \leq \frac{z}{1+n\varepsilon} \mid \mathbf{x}\right) - P\left( \sum_{j \in \mathcal{N}}\frac{t_j}{n} \leq \frac{s}{1-n\varepsilon} \mid \mathbf{x}\right) 
    \leq &G(z\mid \mathbf{x}) - G(s\mid \mathbf{x}) \\
    \leq &P\left( \sum_{j \in \mathcal{N}}\frac{t_j}{n} \leq \frac{z}{1-n\varepsilon} \mid \mathbf{x}\right) - P\left( \sum_{j \in \mathcal{N}}\frac{t_j}{n} \leq \frac{s}{1+n\varepsilon} \mid \mathbf{x}\right).
\end{align}
\normalsize
Then, it follows that
\begin{align}
    b^-(z,\mathbf{x}) = \inf_{s \in \mathcal{Z}}\frac{P(\frac{ns}{1-\varepsilon n} \leq N_T \leq \frac{nz}{1+\varepsilon n}\mid \mathbf{x})}{P(ns\leq N_T \leq nz\mid \mathbf{x})}, 
    \qquad 
    b^+(z,\mathbf{x}) = \sup_{s \in \mathcal{Z}}\frac{P(\frac{ns}{1+\varepsilon n} \leq N_T \leq \frac{nz}{1-\varepsilon n}\mid \mathbf{x})}{P(ns\leq N_T \leq nz\mid \mathbf{x})}.
\end{align}

\circled{2} \textbf{Thresholding function:}

Let $h(t_{\mathcal{N}}) := \sum_{j \in \mathcal{N}}\frac{t_j}{n}$ and assume the exposure mapping is specified  through a threshold as $g(t_{\mathcal{N}}) =  f(h(t_{\mathcal{N}})) := \mathbf{1}_{[h(t_{\mathcal{N}}) \geq c]}$, i.e., $P(g(t_{\mathcal{N}}) = 1 \mid \mathbf{x}) = P(N_T \geq nc \mid \mathbf{x})$, where $N_T$ denotes the number of treated neighbors. We allow the true threshold $c^{\ast}$ to differ by an amount $\varepsilon \in [0,\min\{c, 1-c\}]$ from $c$, i.e., $c^{\ast} \in [c \pm \varepsilon]$. Thus, $P(g^{\ast}(t_{\mathcal{N}}) = 1 \mid \mathbf{x}) = P(N_T \geq nc^{\ast} \mid \mathbf{x})$, and,  therefore, by straightforward computation, we yield
\begin{align}
    \frac{P(N_T \geq n(c+\varepsilon) \mid \mathbf{x})}{P(N_T \geq nc \mid \mathbf{x})} \leq \frac{P(g^{\ast}(t_{\mathcal{N}}) = 1 \mid \mathbf{x})}{P(g(t_{\mathcal{N}})=1 \mid \mathbf{x})} \leq \frac{P(N_T \geq n(c-\varepsilon) \mid \mathbf{x})}{P(N_T \geq nc \mid \mathbf{x})},
\end{align}
and
\begin{align}
    \frac{1-P(N_T \geq n(c-\varepsilon) \mid \mathbf{x})}{1-P(N_T \geq nc \mid \mathbf{x})} \leq \frac{P(g^{\ast}(t_{\mathcal{N}}) = 0 \mid \mathbf{x})}{P(g(t_{\mathcal{N}})=0 \mid \mathbf{x})}\leq \frac{1-P(N_T \geq n(c+\varepsilon) \mid \mathbf{x})}{1-P(N_T \geq nc \mid \mathbf{x})}.
\end{align}

\subsection{Auxiliary theory}
\label{sec:appendix_proof_auxiliary}

Our bounds employ a sensitivity method proposed in \citet{Frauen.2023b}. However, the original contribution proposes bounds in the presence of unobserved confounding, whereas we are targeting a different setting. Below, we present Theorem 1 in \cite{Frauen.2023b} adapted to our setting.

\begin{restatable}{theorem}{gmsm_bounds}\label{thm:gmsm_bounds}
    Let $b^{-}(z,\mathbf{x}) \leq b^{+}(z,\mathbf{x})$ with $b^{-}(z,\mathbf{x}) \in (0,1]$ and $b^{+}(z,\mathbf{x}) \in [1,\infty)$, such that for all $z,\mathbf{x}$
    \begin{align}
        b^{-}(z,\mathbf{x}) \leq \frac{p(g^{\ast}(t_{\mathcal{N}}) = z \mid \mathbf{x})}{p(g(t_{\mathcal{N}})=z \mid \mathbf{x})} \leq b^{+}(z,\mathbf{x})
    \end{align}
    and define $\alpha^{\pm}(z,\mathbf{x}) := \frac{(1-b^{\mp}(z,\mathbf{x}))b^{\pm}(z,\mathbf{x})}{b^{\pm}(z,\mathbf{x}) - b^{\mp}(z,\mathbf{x})}$. Furthermore, let $F_Y(y):=F_Y(y \mid t,z,\mathbf{x})$ denote the conditional cumulative distribution function (CDF) of $Y$. For $Y \in \mathbb{R}$ continuous, we define 
    
    \footnotesize
    \begin{align}
        p^+(y\mid t, z,\mathbf{x}) = \begin{cases}
            \frac{1}{b^+(z,\mathbf{x})}p(y\mid t, z,\mathbf{x}), \emph{ if } \, F(y) \leq \alpha^+(z,\mathbf{x}),\\
            \frac{1}{b^-(z,\mathbf{x})}p(y\mid t, z,\mathbf{x}), \emph{ if } \, F(y) > \alpha^+(z,\mathbf{x}),
        \end{cases}
    \end{align}
    \normalsize
    and for $Y \in \mathbb{R}$ discrete, we define the probability mass function 
    
    \footnotesize
    \begin{align}
        P^+(y\mid t, z,\mathbf{x}) = \begin{cases}
            \frac{1}{b^+(z,\mathbf{x})}P(y\mid t, z,\mathbf{x}),&\emph{ if } \, F(y) < \alpha^+(z,\mathbf{x}),\\
            \frac{1}{b^-(z,\mathbf{x})}P(y\mid t, z,\mathbf{x}), &\emph{ if } \, F(y-1) > \alpha^+(z,\mathbf{x}),\\
            \frac{1}{b^+(z,\mathbf{x})}(\alpha^+(z,\mathbf{x}) - F(y-1)) + \frac{1}{b^-(z,\mathbf{x})}(F(y)-\alpha^+(z,\mathbf{x})) , &\emph{ otherwise.}
        \end{cases}
    \end{align}
    \normalsize
    The lower bound $p^-(y\mid t, z,\mathbf{x})$ is defined through exchanging the signs in $\alpha$ and $b$. Let $F^{\pm}(y)$ denote the conditional CDF with regard to $p^{\pm}(y\mid t, z,\mathbf{x})$. Then, for all $y\in \mathcal{Y}$
    \begin{align}
        F^+(y) \leq \inf_{\Tilde{P} \in \mathcal{M}}F_{\Tilde{P}}(y), F^-(y) \geq \inf_{\Tilde{P} \in \mathcal{M}}F_{\Tilde{P}}(y),
    \end{align}
    i.e., the bounds are valid, and
    \begin{align}
        F^+(y) = \inf_{\Tilde{P} \in \mathcal{M}}F_{\Tilde{P}}(y), F^-(y) = \inf_{\Tilde{P} \in \mathcal{M}}F_{\Tilde{P}}(y),
    \end{align}
    i.e., the bounds are sharp, if $Z$ is continuous or if $Z$ is discrete and $\frac{1}{b^+(z,\mathbf{x})} \geq \pi^g(z\mid\mathbf{x})$. 
\end{restatable}

\subsection{Proof of Theorem \ref{thm:plug_in_bounds}}

\bounds*
\begin{proof}
    Throughout the proof we focus on the upper bound for continuous outcomes. The other cases follow analogously.
    Recall the definition of
\begin{align}
    Q^{\pm}(t,z,\mathbf{x}) := \inf \Bigg \{y \mid F_Y(y \mid t,z,\mathbf{x}) \geq \frac{(1-b^{\mp}(z,\mathbf{x}))b^{\pm}(z,\mathbf{x})}{b^{\pm}(z,\mathbf{x}) - b^{\mp}(z,\mathbf{x})} \Bigg \},
\end{align}
when $b^-(z,\mathbf{x})<1<b^+(z,\mathbf{x})$, and $Q^{\pm}(t,z,\mathbf{x}) = Q(t,z,\mathbf{x}) := \inf \{y \mid F_Y(y \mid t,z,\mathbf{x}) \geq \frac{1}{2}\}$ otherwise.

By applying Theorem~\ref{thm:gmsm_bounds}, the sharp upper and lower bounds on the conditional potential outcome $\mu(t,z,\mathbf{x})$ are given by
\begin{align}
    \mu^{\pm}(t,z,\mathbf{x})= &\frac{1}{b^{\pm}(z,\mathbf{x})}\int_{-\infty}^{Q^{\pm}(z,\mathbf{x})}y \, \mathrm{d}\mu
    + \frac{1}{b^{\mp}(z,\mathbf{x})}\int_{Q^{\pm}(z,\mathbf{x})}^{\infty}y \,\mathrm{d}\mu\\
    = &\frac{1}{b^{\pm}(z,\mathbf{x})} \cdot \alpha^{\pm} \text{LCTE}_{\alpha}^{\pm}(t,z,\mathbf{x})
    + \frac{1}{b^{\mp}(z,\mathbf{x})} \cdot (1-\alpha^{\pm}) \text{CVaR}_{\alpha}^{\pm}(t,z,\mathbf{x})
\end{align}
where we define $\alpha^{\pm} := \frac{(1-b^{\mp}(z,\mathbf{x}))b^{\pm}(z,\mathbf{x})}{b^{\pm}(z,\mathbf{x}) - b^{\mp}(z,\mathbf{x})}$. Here, the $\text{CVaR}^{\pm}$ denotes the \emph{conditional value at risk} at level $\alpha^{\pm}$  with corresponding quantiles $Q^{+}(t,z,\mathbf{x})/Q^{-}(t,z,\mathbf{x})$ defined as

\footnotesize
\begin{align}
    \text{CVaR}_{\alpha}^{+}(t,z,\mathbf{x}):= 
    \min_{q\in\mathbb{R}} \Bigl\{\,q + \frac{1}{1-\alpha^+}\,\mathbb{E}\bigl[\,(Y - q)_+ \mid t,z,\mathbf{x}\bigr]\Bigr\}
    = Q^{+}(t,z,\mathbf{x}) + \frac{b^{-} - b^{+}}{(1-b^{+})b^{-}}\mathbb{E}\bigl[ (Y-Q^{+}(t,z,\mathbf{x}))_{+} \mid t,z,\mathbf{x}\bigr],\\
    \text{CVaR}_{\alpha}^{-}(t,z,\mathbf{x}) := \min_{q\in\mathbb{R}} \Bigl\{\,q + \frac{1}{1-\alpha^-}\,\mathbb{E}\bigl[\,(Y - q)_+ \mid t,z,\mathbf{x}\bigr]\Bigr\} = Q^{-}(t,z,\mathbf{x}) + \frac{b^{+} - b^{-}}{(1-b^{-})b^{+}}\mathbb{E}\bigl[ (Y-Q^{-}(t,z,\mathbf{x}))_{+} \mid t,z,\mathbf{x}\bigr] 
\end{align}
\normalsize
where \((u)_+ = \max\{u,0\}\),
and $\text{LCTE}^{\pm}$ the \emph{lower conditional tail expectation} at level $\alpha^{\pm}$ with corresponding quantiles $Q^{+}(t,z,\mathbf{x}))/Q^{-}(t,z,\mathbf{x}))$ defined as

\footnotesize
\begin{equation}
\begin{aligned}
    \text{LCTE}_{\alpha}^{+}(t,z,\mathbf{x}) :=& \sup_{q\in\mathbb{R}}
        \Bigl\{\,
        q \;-\;\frac{1}{\alpha^+}\,\mathbb{E}\bigl[(\,q - Y\,)_+ \mid t,z,\mathbf{x}\bigr]
        \Bigr\}\\
    =& Q^{+}(t,z,\mathbf{x})) - \frac{b^{+}(z,\mathbf{x}) - b^{-}(z,\mathbf{x})}{(1-b^{-}(z,\mathbf{x}))b^{+}(z,\mathbf{x})}\mathbb{E}\bigl[ (Q^{+}(t,z,\mathbf{x})) - Y)_{+} \mid t,z,\mathbf{x}\bigr],
    \end{aligned}
\end{equation}
\begin{equation}
    \begin{aligned}
    \text{LCTE}_{\alpha}^{-}(t,z,\mathbf{x}) :=& \sup_{q\in\mathbb{R}}
        \Bigl\{\,
        q \;-\;\frac{1}{\alpha^-}\,\mathbb{E}\bigl[(\,q - Y \,)_+ \mid t,z,\mathbf{x}\bigr]
        \Bigr\}\\
    =& Q^{-}(t,z,\mathbf{x})) - \frac{b^{-}(z,\mathbf{x}) - b^{+}(z,\mathbf{x})}{(1-b^{+}(z,\mathbf{x}))b^{-}(z,\mathbf{x})} \mathbb{E}\bigl[ (Q^{-}(t,z,\mathbf{x})) - Y)_{+} \mid t,z,\mathbf{x}\bigr].
    \end{aligned}
\end{equation}
\normalsize
With these reformulations of CVaR and LCTE then follows the desired result
\begin{align}
        \mu^{\pm}(t,z, \mathbf{x}) = Q^{\pm}(t,z,\mathbf{x}) + \frac{1}{b^{\mp}(z,\mathbf{x})}\mathbb{E}\bigl[(Y-Q^{\pm}(t,z,\mathbf{x}))_{+}\mid t,z,\mathbf{x}\bigr] - \frac{1}{b^{\pm}(z,\mathbf{x})}\mathbb{E}\bigl[ (Q^{\pm}(t,z,\mathbf{x})-Y)_{+} \mid t,z,\mathbf{x} \bigr].
\end{align}
\end{proof}

\subsection{Proof of Theorem \ref{thm:orthogonal_bounds}}

{
\renewcommand{\vspace}[1]{}
\renewcommand{\addvspace}[1]{}
\orthobounds*
}

\begin{proof}
We begin with the case where $Z$ is discrete (including binary), so $\omega_{z,h}(Z)=\mathbf{1}_{[Z=z]}$.
We discuss the continuous-$Z$ modification at the end.

Fix $(t,z)$ and abbreviate
\begin{align}
p(t,z\mid \mathbf{X}) := \pi^t(\mathbf{X})\,\pi^g(z\mid \mathbf{X}),\qquad
\alpha := \alpha^+(z,\mathbf{X}),\qquad
Q := Q^+(t,z,\mathbf{X}).
\end{align}
Define the (unnormalized) conditional tail moments
\begin{align}
\gamma_u := \gamma_u^+(t,z,\mathbf{X})
:= \mathbb{E}\!\left[(Y-Q)_+ \mid T=t,Z=z,\mathbf{X}\right],\qquad
\gamma_l := \gamma_l^+(t,z,\mathbf{X})
:= \mathbb{E}\!\left[(Q-Y)_+ \mid T=t,Z=z,\mathbf{X}\right].
\end{align}
The sharp upper bound can be written as
\begin{equation}\label{eq:muplus_gamma_form}
\mu^+(t,z,\mathbf{X})
=
Q + \frac{\gamma_u}{b^{-}(z,\mathbf{X})}
- \frac{\gamma_l}{b^{+}(z,\mathbf{X})},
\end{equation}
which matches the first line of Eq.~\eqref{eq:pseudo_outcome_unified} when $\widehat\eta=\eta$.

\paragraph{Step 1: Reparameterization of $\mu^+$ as a convex combination of CVaR/LCTE functionals.}
Define the upper-tail and lower-tail pseudo-outcomes at level $\alpha$ (see, e.g., \citet{Dorn.2025,Oprescu.2023})
\begin{align}
H_u(y,q) := q + \frac{1}{1-\alpha}(y-q)_+,
\qquad
H_l(y,q) := q - \frac{1}{\alpha}(q-y)_+.
\end{align}
Their conditional expectations at the true quantile $Q$ are the conditional upper CVaR and lower conditional tail expectation (LCTE), respectively:
\begin{align}
\theta_u(\mathbf{X})
:= \mathbb{E}\!\left[H_u(Y,Q)\mid T=t,Z=z,\mathbf{X}\right]
= Q + \frac{1}{1-\alpha}\gamma_u,
\qquad
\theta_l(\mathbf{X})
:= \mathbb{E}\!\left[H_l(Y,Q)\mid T=t,Z=z,\mathbf{X}\right]
= Q - \frac{1}{\alpha}\gamma_l.
\end{align}
Now set the weights
\begin{align}
w_u(\mathbf{X}) := \frac{1-\alpha}{b^{-}(z,\mathbf{X})},
\qquad
w_l(\mathbf{X}) := \frac{\alpha}{b^{+}(z,\mathbf{X})}.
\end{align}
By the definition of $\alpha^+(z,\mathbf{X})$, one has
\begin{equation}\label{eq:alpha_identity}
w_u(\mathbf{X}) + w_l(\mathbf{X})
=
\frac{1-\alpha}{b^{-}(z,\mathbf{X})} + \frac{\alpha}{b^{+}(z,\mathbf{X})}
= 1.
\end{equation}
Therefore,
\begin{align}
w_u(\mathbf{X})\theta_u(\mathbf{X}) + w_l(\mathbf{X})\theta_l(\mathbf{X})
=
(w_u+w_l)Q + \frac{w_u}{1-\alpha}\gamma_u - \frac{w_l}{\alpha}\gamma_l
=
Q + \frac{1}{b^{-}}\gamma_u - \frac{1}{b^{+}}\gamma_l
=
\mu^+(t,z,\mathbf{X}),
\end{align}
so $\mu^+$ is a (convex) linear combination of the two tail functionals.

\paragraph{Step 2: Recentered efficient influence function for $\mu^+$.}
Our orthogonal pseudo-outcome is the recentered efficient influence function (REIF) of $\mu^+(t,z,\mathbf{X})$.
Since $w_u(\mathbf{X}),w_l(\mathbf{X})$ are known functions of $(b^\pm,\alpha)$ (hence fixed with respect to the data-generating distribution),
linearity of REIFs implies
\begin{equation}\label{eq:reif_linearity}
\phi^+_{t,z}(S;\eta)
:= \mathrm{REIF}(\mu^+(t,z,\mathbf{X}))
=
w_u(\mathbf{X})\,\phi_u(S;\eta) + w_l(\mathbf{X})\,\phi_l(S;\eta),
\end{equation}
where $\phi_u(S;\eta):=\mathrm{REIF}(\theta_u(\mathbf{X}))$ and $\phi_l(S;\eta):=\mathrm{REIF}(\theta_l(\mathbf{X}))$.

Define the selection weight
\begin{align}
\kappa_{t,z}(S)
:=
\frac{\mathbf{1}_{[T=t]}\,\mathbf{1}_{[Z=z]}}{\pi^t(\mathbf{X})\,\pi^g(z\mid \mathbf{X})}.
\end{align}
By the known REIFs for conditional CVaR/LCTE functionals (e.g., \citet{Dorn.2025,Oprescu.2023}),
\begin{equation}\label{eq:reif_thetas}
\phi_u(S;\eta)=\theta_u(\mathbf{X})+\kappa_{t,z}(S)\bigl(H_u(Y,Q)-\theta_u(\mathbf{X})\bigr),\quad
\phi_l(S;\eta)=\theta_l(\mathbf{X})+\kappa_{t,z}(S)\bigl(H_l(Y,Q)-\theta_l(\mathbf{X})\bigr).
\end{equation}
Moreover, these REIFs are \emph{orthogonal with respect to} $Q$: the cutoff $Q$ is characterized as the optimizer of the
corresponding tail objective (equivalently, the Rockafellar--Uryasev CVaR variational form), so the envelope/first-order
condition yields $\partial_q \mathbb{E}[H_u(Y,q)\mid t,z,\mathbf{X}]|_{q=Q}=0$ and
$\partial_q \mathbb{E}[H_l(Y,q)\mid t,z,\mathbf{X}]|_{q=Q}=0$ (see \citet{Dorn.2025,Oprescu.2023}).

Finally, substituting \eqref{eq:reif_thetas} into \eqref{eq:reif_linearity}, using
$\theta_u(\mathbf{X})=Q+\gamma_u/(1-\alpha)$ and $\theta_l(\mathbf{X})=Q-\gamma_l/\alpha$, and simplifying with
$w_u/(1-\alpha)=1/b^{-}$ and $w_l/\alpha=1/b^{+}$ yields exactly Eq.~\eqref{eq:pseudo_outcome_unified}.

\paragraph{Step 3: Unbiasedness and orthogonality.}
Orthogonality (Neyman-orthogonality) follows because $\phi^+_{t,z}$ is a linear combination of orthogonal REIFs for
$\theta_u$ and $\theta_l$ (linearity preserves orthogonality), and because $\theta_u,\theta_l$ themselves are orthogonal both to
the selection nuisance $(\pi^t,\pi^g)$ and to the regression nuisances via the standard conditional-mean EIF from Eq.~(\eqref{eq:reif_thetas}). Orthogonality with respect to $Q$ is guaranteed by the envelope/first-order condition (FOC) argument above.

Unbiasedness follows by iterated expectations: conditional on $\mathbf{X}$,
{
\small
\begin{align}
\mathbb{E}\!\left[
\frac{\mathbf{1}_{[T=t]}\mathbf{1}_{[Z=z]}}{p(t,z\mid\mathbf{X})}
\left\{ \frac{(Y-Q)_+-\gamma_u}{b^{-}} - \frac{(Q-Y)_+-\gamma_l}{b^{+}}\right\}
\Bigm|\mathbf{X}\right]
=
\mathbb{E}\!\left[\frac{(Y-Q)_+-\gamma_u}{b^{-}}-\frac{(Q-Y)_+-\gamma_l}{b^{+}}\Bigm|T=t,Z=z,\mathbf{X}\right]
=0,
\end{align}
\normalsize
}
so $\mathbb{E}[\phi^+_{t,z}(S;\eta)\mid \mathbf{X}]=\mu^+(t,z,\mathbf{X})$. This completes the proof for discrete $Z$.

\paragraph{Continuous $Z$.}
When $Z$ is continuous, evaluation at $Z=z$ is not pathwise differentiable.
We instead use kernel localization: replace $\mathbf{1}_{[Z=z]}$ in $\kappa_{t,z}$ by $\omega_{z,h}(Z)=K_h(Z-z)$ and replace the pmf
$\pi^g(z\mid \mathbf{X})$ by the conditional density $\pi^g(Z\mid \mathbf{X})$ to define the localized weight
\begin{align}
\kappa_{t,z,h}(S):=\frac{\mathbf{1}_{[T=t]}\,K_h(Z-z)}{\pi^t(\mathbf{X})\,\pi^g(Z\mid \mathbf{X})}.
\end{align}
Then Eq.~\eqref{eq:reif_thetas} and the linearity relation from Eq.~\eqref{eq:reif_linearity} hold verbatim with $\kappa_{t,z}$ replaced by
$\kappa_{t,z,h}$, yielding the localized pseudo-outcome in Eq.~\eqref{eq:pseudo_outcome_unified}. The same iterated-expectations
argument gives $\mathbb{E}[\phi^+_{t,z,h}(S;\eta)\mid \mathbf{X}]=\mu_h^+(t,z,\mathbf{X})$, and under standard smoothness in $z$,
$\mu_h^+(t,z,\mathbf{X})\to \mu^+(t,z,\mathbf{X})$ as $h\downarrow 0$.
\end{proof}

\subsection{Proof of Theorem \ref{thm:second_order_remainder}}

{
\renewcommand{\vspace}[1]{}
\renewcommand{\addvspace}[1]{}
\nuisanceReminder*
}

\begin{proof}
We prove the statement for discrete $Z$. Throughout, fix $(t,z)$ and suppress $(t,z)$ in the notation whenever clear.
Because we use $K$-fold cross-fitting, for any observation in a held-out fold the nuisance estimates
$\widehat\eta=(\widehat\pi^t,\widehat\pi^g,\widehat Q^+,\widehat\gamma_u^+,\widehat\gamma_l^+)$ are functions of the
training folds only; hence, when taking expectations over the held-out fold, we may treat $\widehat\eta$ as fixed (formally,
condition on the training sample).

Let $A:=\mathbf{1}_{[T=t]}\mathbf{1}_{[Z=z]}$ and write the true and estimated joint propensities as
\begin{align}
\pi(\mathbf{X}) := \pi^t(\mathbf{X})\pi^g(z\mid\mathbf{X}),\qquad
\widehat\pi(\mathbf{X}) := \widehat\pi^t(\mathbf{X})\widehat\pi^g(z\mid\mathbf{X}).
\end{align}
Also denote the (population) conditional means at an arbitrary cutoff $\widehat Q$:
\begin{align}
\gamma_u(\widehat Q; \mathbf{X}) := \mathbb{E}\!\left[(Y-\widehat Q(\mathbf{X}))_+\mid T=t,Z=z,\mathbf{X}\right],\quad
\gamma_l(\widehat Q; \mathbf{X}) := \mathbb{E}\!\left[(\widehat Q(\mathbf{X})-Y)_+\mid T=t,Z=z,\mathbf{X}\right] 
\end{align}
with $\gamma_u(Q^+; \mathbf{X}) = \gamma_u(\mathbf{X})$ and $\gamma_l(Q^+; \mathbf{X}) = \gamma_l(\mathbf{X})$.

\paragraph{Step 1: Conditional expectation of the estimated pseudo-outcome.}
For discrete $Z$, the pseudo-outcome simplifies (since $A$ forces $Z=z$ inside the square bracket) to
\begin{align}
\phi(S;\widehat\eta)
&=
\widehat Q(\mathbf{X})
+\frac{\widehat\gamma_u(\mathbf{X})}{b^-(z,\mathbf{X})}
-\frac{\widehat\gamma_l(\mathbf{X})}{b^+(z,\mathbf{X})}
+\frac{A}{\widehat\pi(\mathbf{X})}
\Bigg[
\frac{(Y-\widehat Q(\mathbf{X}))_+ - \widehat\gamma_u(\mathbf{X})}{b^-(z,\mathbf{X})}
-\frac{(\widehat Q(\mathbf{X})-Y)_+ - \widehat\gamma_l(\mathbf{X})}{b^+(z,\mathbf{X})}
\Bigg].
\end{align}
Taking conditional expectations given $\mathbf{X}$ and using $\mathbb{E}[A\mid \mathbf{X}] = \pi(\mathbf{X})$ yields
\begin{align}
\mathbb{E}\!\left[\phi(S;\widehat\eta)\mid \mathbf{X}\right]
&=
\widehat Q(\mathbf{X})
+\frac{\widehat\gamma_u(\mathbf{X})}{b^-(z,\mathbf{X})}
-\frac{\widehat\gamma_l(\mathbf{X})}{b^+(z,\mathbf{X})}
+\frac{\pi(\mathbf{X})}{\widehat\pi(\mathbf{X})}
\Bigg[
\frac{\gamma_u(\widehat Q^+; \mathbf{X}) - \widehat\gamma_u(\mathbf{X})}{b^-(z,\mathbf{X})}
-\frac{\gamma_l(\widehat Q^+; \mathbf{X}) - \widehat\gamma_l(\mathbf{X})}{b^+(z,\mathbf{X})}
\Bigg]\\
& = \underbrace{\widehat Q(\mathbf{X})
+\frac{\gamma_u(\widehat Q^+; \mathbf{X})}{b^-(z,\mathbf{X})}
-\frac{\gamma_l(\widehat Q^+; \mathbf{X})}{b^+(z,\mathbf{X})}}_{=:~\mu_{\widehat Q^+}(\mathbf{X})}
+\left(\frac{\pi(\mathbf{X})}{\widehat\pi(\mathbf{X})}-1\right)
\Bigg[
\frac{\gamma_u(\widehat Q^+; \mathbf{X}) - \widehat\gamma_u(\mathbf{X})}{b^-(z,\mathbf{X})}
-\frac{\gamma_l(\widehat Q^+; \mathbf{X}) - \widehat\gamma_l(\mathbf{X})}{b^+(z,\mathbf{X})}
\Bigg]
\end{align}

Moreover, by Theorem~\ref{thm:orthogonal_bounds} (applied with true nuisances), we yield
\begin{align}
\mathbb{E}\!\left[\phi(S;\eta)\mid \mathbf{X}\right] = \mu^+(\mathbf{X})
:= Q(\mathbf{X}) + \frac{\gamma_u(\mathbf{X})}{b^-(z,\mathbf{X})} - \frac{\gamma_l(\mathbf{X})}{b^+(z,\mathbf{X})}.
\end{align}

Thus, we arrive at
\begin{align}
\mathbb{E}\!\left[\phi(S;\widehat\eta)-\phi(S;\eta)\mid \mathbf{X}\right]
&=
\mu_{\widehat Q^+}(\mathbf{X})-\mu^+(\mathbf{X})
+
\left(\frac{\pi(\mathbf{X})}{\widehat\pi(\mathbf{X})}-1\right)
\Bigg[
\frac{\gamma_u(\widehat Q^+; \mathbf{X}) - \widehat\gamma_u(\mathbf{X})}{b^-(z,\mathbf{X})}
-\frac{\gamma_l(\widehat Q^+; \mathbf{X}) - \widehat\gamma_l(\mathbf{X})}{b^+(z,\mathbf{X})}
\Bigg],\label{eq:remainder_decomp}
\end{align}
and
\begin{align}
    \left\| \mathbb{E}\!\left[\phi(S;\widehat\eta)-\phi(S;\eta)\mid \mathbf{X}\right]\right\|_2\le 
    \underbrace{\left\| \mu_{\widehat Q}(\mathbf{X})-\mu^+(\mathbf{X})\right\|_2}_{\text{cutoff-induced error}} + \underbrace{\left\| \frac{\pi(\mathbf{X})}{\widehat\pi(\mathbf{X})}-1\right\|_2 \left\|
\frac{\gamma_u(\widehat Q^+; \mathbf{X}) - \widehat\gamma_u(\mathbf{X})}{b^-(z,\mathbf{X})}
-\frac{\gamma_l(\widehat Q^+; \mathbf{X}) - \widehat\gamma_l(\mathbf{X})}{b^+(z,\mathbf{X})}
\right\|_2}_{\text{propensity $\times$ regression product term}}
\end{align}
where the last inequality is due to the triangle inequality and Cauchy–Schwarz inequality.

\paragraph{Step 2: Bounding the product term by $O_p(r_{n,\pi}r_{n,\gamma})$.}
By Assumption~\ref{assump:overlap_bounded}, $\widehat\pi(\mathbf{X})\ge \varepsilon$ a.s., hence
\begin{align}
\left\|\frac{\pi(\mathbf{X})}{\widehat\pi(\mathbf{X})}-1\right\|_2
=\left\|\frac{\pi(\mathbf{X})-\widehat\pi(\mathbf{X})}{\widehat\pi(\mathbf{X})}\right\|_2
\le \varepsilon^{-1}\|\widehat\pi-\pi\|_2.
\end{align}
Since $\pi=\pi^t\pi^g$ and $\widehat\pi=\widehat\pi^t\widehat\pi^g$,
\begin{align}
\widehat\pi-\pi=(\widehat\pi^t-\pi^t)\widehat\pi^g+\pi^t(\widehat\pi^g-\pi^g),
\end{align}
so by the triangle inequality and $0\le \pi^t,\widehat\pi^g\le 1$ (discrete $Z$),
\begin{align}
\|\widehat\pi-\pi\|_2 \le \|\widehat\pi^t-\pi^t\|_2+\|\widehat\pi^g-\pi^g\|_2.
\end{align}
Therefore
\begin{align}
\left\|\frac{\pi}{\widehat\pi}-1\right\|_2
\le \varepsilon^{-1}\Big(\|\widehat\pi^t-\pi^t\|_2+\|\widehat\pi^g-\pi^g\|_2\Big)
\end{align}
and it remains to bound the second factor. Since $b^{-}(z,\mathbf{X})$ is bounded away from $0$, we have
\begin{align}
\left\|
\frac{\gamma_u(\widehat Q^+; \mathbf{X}) - \widehat\gamma_u(\mathbf{X})}{b^-(z,\mathbf{X})}
-\frac{\gamma_l(\widehat Q^+; \mathbf{X}) - \widehat\gamma_l(\mathbf{X})}{b^+(z,\mathbf{X})}
\right\|_2
\le
\varepsilon^{-1}\Big(
\|\widehat\gamma_u-\gamma_u(\widehat Q^+;\cdot)\|_2
+
\|\widehat\gamma_l-\gamma_l(\widehat Q^+;\cdot)\|_2
\Big),
\end{align}
by the triangle inequality. Combining with the previous inequality yields
\begin{align}
\underbrace{\left\|\frac{\pi}{\widehat\pi}-1\right\|_2}_{=\,O_p(r_{n,\pi})}\cdot
\underbrace{\left\|
\frac{\gamma_u(\widehat Q^+; \mathbf{X}) - \widehat\gamma_u(\mathbf{X})}{b^-}
-\frac{\gamma_l(\widehat Q^+; \mathbf{X}) - \widehat\gamma_l(\mathbf{X})}{b^+}
\right\|_2}_{=\,O_p(r_{n,\gamma})}
=O_p(r_{n,\pi}r_{n,\gamma}),
\end{align}
where the second $O_p(r_{n,\gamma})$ is by definition of $r_{n,\gamma}$ (as the $L_2$ rate for estimating the conditional tail means at the cutoff used in the pseudo-outcome).

\paragraph{Step 3: Bounding the cutoff-induced term by $O_p(r_{n,Q}^2)$.}
We now need to control the term $\mu_{\widehat Q^+}(\mathbf{X})-\mu^+(\mathbf{X})$.
Fix $(t,z)$ and $\mathbf{x}$, and define the scalar function
\begin{align}
\mathcal{L}_{\mathbf{x}}(q)
:= q
+ \frac{1}{b^-(z,\mathbf{x})}\,\mathbb{E}\!\left[(Y-q)_+\mid T=t,Z=z,\mathbf{X}=\mathbf{x}\right]
-\frac{1}{b^+(z,\mathbf{x})}\,\mathbb{E}\!\left[(q-Y)_+\mid T=t,Z=z,\mathbf{X}=\mathbf{x}\right].
\end{align}
By construction,
\begin{align}
\mu_{\widehat Q^+}(\mathbf{x})=\mathcal{L}_{\mathbf{x}}(\widehat Q^+(t,z,\mathbf{x})),
\qquad
\mu^+(\mathbf{x})=\mathcal{L}_{\mathbf{x}}(Q^+(t,z,\mathbf{x})),
\end{align}
where $Q^+(t,z,\mathbf{x})$ is the optimal cutoff from Theorem~\ref{thm:plug_in_bounds}.

Assume (as is standard for quantile/CVaR-style expansions) the conditional CDF $F_{Y\mid t,z,\mathbf{x}}$ is differentiable in a neighborhood of $Q^+(t,z,\mathbf{x})$
with density $f_{Y\mid t,z,\mathbf{x}}$ bounded by $\bar f<\infty$. Then, $\mathcal{L}_{\mathbf{x}}$ is differentiable and
\begin{align}
\mathcal{L}'_{\mathbf{x}}(q)
= 1 - \frac{1-F_{Y\mid t,z,\mathbf{x}}(q)}{b^-(z,\mathbf{x})}
-\frac{F_{Y\mid t,z,\mathbf{x}}(q)}{b^+(z,\mathbf{x})}.
\end{align}
Moreover, $\mathcal{L}'_{\mathbf{x}}$ is Lipschitz with
\begin{align}
|\mathcal{L}''_{\mathbf{x}}(q)|
= \left|\left(\frac{1}{b^-(z,\mathbf{x})}-\frac{1}{b^+(z,\mathbf{x})}\right) f_{Y\mid t,z,\mathbf{x}}(q)\right|
\le \bar f\left(\frac{1}{b^-(z,\mathbf{x})}+\frac{1}{b^+(z,\mathbf{x})}\right)
\le L,
\end{align}
for a finite constant $L$ (uniform in $\mathbf{x}$ by Assumption~\ref{assump:overlap_bounded}).

By optimality of $Q^+(t,z,\mathbf{x})$ for $\mathcal{L}_{\mathbf{x}}$, we have
$\mathcal{L}'_{\mathbf{x}}(Q^+(t,z,\mathbf{x}))=0$. Therefore, by the fundamental theorem of calculus,
\begin{align*}
\big|\mathcal{L}_{\mathbf{x}}(\widehat Q^+(t,z,\mathbf{x}))-\mathcal{L}_{\mathbf{x}}(Q^+(t,z,\mathbf{x}))\big|
&=
\left|\int_{Q^+(t,z,\mathbf{x})}^{\widehat Q^+(t,z,\mathbf{x})}
\big(\mathcal{L}'_{\mathbf{x}}(u)-\mathcal{L}'_{\mathbf{x}}(Q^+(t,z,\mathbf{x}))\big)\,\diff u\right|\\
&\le
\int_{Q^+(t,z,\mathbf{x})}^{\widehat Q^+(t,z,\mathbf{x})} L\,|u-Q^+(t,z,\mathbf{x})|\,\diff u\\
&\le \frac{L}{2}\,|\widehat Q^+(t,z,\mathbf{x})-Q^+(t,z,\mathbf{x})|^2.
\end{align*}
Taking $L_2(P_{\mathbf{X}})$ norms yields
\begin{align}
\|\mu_{\widehat Q^+}-\mu^+\|_2
= O_p\!\left(\|\widehat Q^+-Q^+\|_2^2\right)
= O_p(r_{n,Q}^2).
\end{align}
\paragraph{Conclusion.}
Combining Step 2 and Step 3 in the decomposition from Eq.~\eqref{eq:remainder_decomp} yields
\begin{align}
\left\|\mathbb{E}\!\left[\phi^+_{t,z}(S;\widehat\eta)-\phi^+_{t,z}(S;\eta)\mid \mathbf{X}\right]\right\|_2
=O_p\!\left(r_{n,\pi}\,r_{n,\gamma}+r_{n,Q}^2\right),
\end{align}
which is exactly Eq.~\eqref{eq:remainder_bound}.
\end{proof}


\subsection{Proof of Corollary \ref{cor:quasi_oracle_and_inference}}

{
\renewcommand{\vspace}[1]{}
\renewcommand{\addvspace}[1]{}
\quasiOracleRates*
}

\begin{proof}
We prove the CAPO and APO statements for the upper bound; the lower-bound case follows by the same argument with the sign-swapped pseudo-outcome.

\paragraph{CAPO bound rate.}
Let $m^+_{t,z}(\mathbf{x}):=\mathbb{E}[\widehat\phi^+_{t,z}\mid \mathbf{X}=\mathbf{x}]$ denote the conditional mean of
the (cross-fitted) pseudo-outcome. By Assumption~\ref{assump:second_stage},
\begin{align}
\|\widehat\mu^+(t,z,\cdot)-m^+_{t,z}(\cdot)\|_2 = O_p(\delta_n).
\end{align}
By the triangle inequality,
\begin{align}
\|\widehat\mu^+(t,z,\cdot)-\mu^+(t,z,\cdot)\|_2
\le
\|\widehat\mu^+(t,z,\cdot)-m^+_{t,z}(\cdot)\|_2
+\|m^+_{t,z}(\cdot)-\mu^+(t,z,\cdot)\|_2.
\end{align}
The second term is precisely the conditional bias induced by nuisance estimation. Applying
Theorem~\ref{thm:second_order_remainder} yields
\begin{align}
\|m^+_{t,z}(\cdot)-\mu^+(t,z,\cdot)\|_2
=
O_p(r_{n,\pi}\,r_{n,\gamma}+r_{n,Q}^2).
\end{align}
Combining the two bounds gives the stated CAPO rate.

\paragraph{APO rate and asymptotic normality.}
Recall $\widehat\psi^+(t,z)=\mathbb{E}_n[\widehat\phi^+_{t,z}]$. Decompose
\begin{align}
\widehat\psi^+(t,z)-\psi^+(t,z)
=
(\mathbb{E}_n-\mathbb{E})[\phi^+_{t,z}(S;\eta)]
+\mathbb{E}\!\left[\phi^+_{t,z}(S;\widehat\eta)-\phi^+_{t,z}(S;\eta)\right]
+R_n,
\end{align}
where $R_n:=(\mathbb{E}_n-\mathbb{E})[\phi^+_{t,z}(S;\widehat\eta)-\phi^+_{t,z}(S;\eta)]$ is an empirical-process term.
Under Assumption~\ref{assump:overlap_bounded}(iii), $\phi^+_{t,z}(S;\cdot)$ is uniformly bounded, and with cross-fitting
$R_n=o_p(n^{-1/2})$ by standard arguments (conditioning on training folds and applying Hoeffding/Bernstein inequalities).

The first term is $O_p(n^{-1/2})$ by the CLT.
The second term is bounded by Theorem~\ref{thm:second_order_remainder} (after integrating over $\mathbf{X}$), yielding
\begin{align}
\left|\mathbb{E}\!\left[\phi^+_{t,z}(S;\widehat\eta)-\phi^+_{t,z}(S;\eta)\right]\right|
=O_p(r_{n,\pi}\,r_{n,\gamma}+r_{n,Q}^2).
\end{align}
This proves Eq.~\eqref{eq:apo_asymp_linear}. If additionally $r_{n,\pi}r_{n,\gamma}+r_{n,Q}^2=o_p(n^{-1/2})$, then the nuisance-induced bias term and $R_n$ are
$o_p(n^{-1/2})$, hence
\begin{align}
\sqrt{n}\big(\widehat\psi^+(t,z)-\psi^+(t,z)\big)
=
\sqrt{n}(\mathbb{E}_n-\mathbb{E})[\phi^+_{t,z}(S;\eta)]
+o_p(1)
\ \rightsquigarrow\
\mathcal{N}(0,V^+(t,z)),
\end{align}
with $V^+(t,z)=\mathrm{Var}(\phi^+_{t,z}(S;\eta))$.
\end{proof}


\subsection{Proof of Proposition \ref{prop:sharpness}}

\sharpness*

\begin{proof}
We show the claim for the CAPO upper bound; the other bounds (CAPO lower, APO upper/lower) follow similarly.

By Corollary~\ref{cor:quasi_oracle_and_inference},
\begin{align}
\|\widehat\mu^+(t,z,\cdot)-\mu^+(t,z,\cdot)\|_2
=O_p\!\left(\delta_n+r_{n,\pi}r_{n,\gamma}+r_{n,Q}^2\right).
\end{align}
Under the proposition assumptions, $\delta_n=o_p(1)$ and $r_{n,Q}=o_p(1)$. Moreover, if either $r_{n,\pi}=o_p(1)$ or
$r_{n,\gamma}=o_p(1)$, then $r_{n,\pi}r_{n,\gamma}=o_p(1)$. Therefore, the right-hand side is $o_p(1)$, implying
$\|\widehat\mu^+(t,z,\cdot)-\mu^+(t,z,\cdot)\|_2=o_p(1)$.

For APOs, the corresponding statement follows from the APO rate in
Corollary~\ref{cor:quasi_oracle_and_inference} and the same convergence of the remainder.
Finally, repeating the same argument for the lower bound (using its analogous pseudo-outcome) establishes convergence of
both endpoints and, hence, convergence of the estimated intervals to the sharp identified intervals.
\end{proof}


\subsection{Proof of Corollary \ref{cor:validity}}

{
\renewcommand{\vspace}[1]{}
\renewcommand{\addvspace}[1]{}
\validity*
}

\begin{proof}
We prove the CAPO claim; the APO claim follows by taking expectations over $\mathbf{X}$.

\paragraph{Step 1: Any cutoff induces a valid (conservative) interval.}
Fix $(t,z,\mathbf{x})$ and define for any scalar cutoff $q$ the upper and lower tail objectives
\begin{align}
\mathcal{L}^+_{\mathbf{x}}(q)
:= q + \frac{1}{b^-(z,\mathbf{x})}\mathbb{E}\!\left[(Y-q)_+\mid T=t,Z=z,\mathbf{X}=\mathbf{x}\right]
-\frac{1}{b^+(z,\mathbf{x})}\mathbb{E}\!\left[(q-Y)_+\mid T=t,Z=z,\mathbf{X}=\mathbf{x}\right],
\end{align}
\begin{align}
\mathcal{L}^-_{\mathbf{x}}(q)
:= q + \frac{1}{b^+(z,\mathbf{x})}\mathbb{E}\!\left[(Y-q)_+\mid T=t,Z=z,\mathbf{X}=\mathbf{x}\right]
-\frac{1}{b^-(z,\mathbf{x})}\mathbb{E}\!\left[(q-Y)_+\mid T=t,Z=z,\mathbf{X}=\mathbf{x}\right].
\end{align}
By Theorem~\ref{thm:plug_in_bounds} (equivalently, the standard Rockafellar--Uryasev variational form),
\begin{align}
\mu^+(t,z,\mathbf{x})=\inf_{q}\mathcal{L}^+_{\mathbf{x}}(q),
\qquad
\mu^-(t,z,\mathbf{x})=\sup_{q}\mathcal{L}^-_{\mathbf{x}}(q),
\end{align}
with optimizers $q=Q^+(t,z,\mathbf{x})$ and $q=Q^-(t,z,\mathbf{x})$.
Hence, for any measurable $\overline Q^+(t,z,\mathbf{x})$ and $\overline Q^-(t,z,\mathbf{x})$,
\begin{align}
\overline\mu^+(t,z,\mathbf{x};\overline Q^+):=\mathcal{L}^+_{\mathbf{x}}(\overline Q^+(t,z,\mathbf{x}))\ \ge\ \mu^+(t,z,\mathbf{x}),
\qquad
\overline\mu^-(t,z,\mathbf{x};\overline Q^-):=\mathcal{L}^-_{\mathbf{x}}(\overline Q^-(t,z,\mathbf{x}))\ \le\ \mu^-(t,z,\mathbf{x}),
\end{align}
so $[\overline\mu^-(t,z,\mathbf{x}),\overline\mu^+(t,z,\mathbf{x})]$ contains the sharp CAPO interval and is therefore valid.

\paragraph{Step 2: Convergence to the induced bounds.}
Fix measurable cutoffs $\overline Q^\pm$ and define the induced hinge-mean targets
\begin{align}
\overline\gamma_u^\pm(t,z,\mathbf{x})
&:=\mathbb{E}\!\left[(Y-\overline Q^\pm(t,z,\mathbf{x}))_+\mid T=t,Z=z,\mathbf{X}=\mathbf{x}\right],\\
\overline\gamma_l^\pm(t,z,\mathbf{x})
&:=\mathbb{E}\!\left[(\overline Q^\pm(t,z,\mathbf{x})-Y)_+\mid T=t,Z=z,\mathbf{X}=\mathbf{x}\right].
\end{align}
Let $\overline\eta^\pm:=(\pi^t,\pi^g,\overline Q^\pm,\overline\gamma_u^\pm,\overline\gamma_l^\pm)$.
By Theorem~\ref{thm:orthogonal_bounds}, the corresponding pseudo-outcome is conditionally unbiased:
$\mathbb{E}[\phi^\pm_{t,z}(S;\overline\eta^\pm)\mid \mathbf{X}]=\overline\mu^\pm(t,z,\mathbf{X};\overline Q^\pm)$.

Now consider the estimated pseudo-outcome $\phi^\pm_{t,z}(S;\widehat\eta^\pm)$ and write $\widehat Q=\widehat Q^\pm$,
$\overline Q=\overline Q^\pm$ for brevity. The same conditional-expectation algebra as in the proof of
Theorem~\ref{thm:second_order_remainder} yields the decomposition
\begin{align}
\mathbb{E}[\phi^\pm_{t,z}(S;\widehat\eta^\pm)\mid \mathbf{X}]
=
\mu_{\widehat Q}^\pm(\mathbf{X})
+\left(\frac{\pi(\mathbf{X})}{\widehat\pi(\mathbf{X})}-1\right)\Delta_{\widehat Q}^\pm(\mathbf{X}),
\end{align}
where $\mu_{\widehat Q}^\pm(\mathbf{X})$ is the induced bound functional evaluated at $\widehat Q$ (i.e., $\mathcal{L}^\pm_{\mathbf{X}}(\widehat Q)$)
and $\Delta_{\widehat Q}^\pm(\mathbf{X})$ collects the conditional-mean regression errors at cutoff $\widehat Q$.

First, since $(u)_+$ is 1-Lipschitz, for each $\mathbf{X}$,
\begin{align}
|\gamma_u(\widehat Q;\mathbf{X})-\gamma_u(\overline Q;\mathbf{X})|
\le |\widehat Q(\mathbf{X})-\overline Q(\mathbf{X})|,
\qquad
|\gamma_l(\widehat Q;\mathbf{X})-\gamma_l(\overline Q;\mathbf{X})|
\le |\widehat Q(\mathbf{X})-\overline Q(\mathbf{X})|.
\end{align}
Using $b^\pm$ bounded away from $0$, this implies $\|\mu_{\widehat Q}^\pm-\overline\mu^\pm(\cdot;\overline Q)\|_2
\lesssim \|\widehat Q-\overline Q\|_2=o_p(1)$ whenever $\widehat Q\to \overline Q$ in $L_2$.

Second, for the product term, Assumption~\ref{assump:overlap_bounded} implies
$\|\pi/\widehat\pi-1\|_2$ is bounded, and, if $(\widehat\pi^t,\widehat\pi^g)$ is consistent, then
$\|\pi/\widehat\pi-1\|_2=o_p(1)$. Moreover,
\begin{align}
\|\gamma_u(\widehat Q;\cdot)-\widehat\gamma_u^\pm\|_2
\le \|\overline\gamma_u^\pm-\widehat\gamma_u^\pm\|_2+\|\gamma_u(\widehat Q;\cdot)-\gamma_u(\overline Q;\cdot)\|_2
\le \|\overline\gamma_u^\pm-\widehat\gamma_u^\pm\|_2+\|\widehat Q-\overline Q\|_2,
\end{align}
and similarly for the lower hinge mean. Hence, if $(\widehat\gamma_u^\pm,\widehat\gamma_l^\pm)$ is consistent for the induced targets
$(\overline\gamma_u^\pm,\overline\gamma_l^\pm)$ and $\widehat Q\to \overline Q$, then $\|\Delta_{\widehat Q}^\pm\|_2=o_p(1)$,
so the product term is $o_p(1)$ even if $(\widehat\pi^t,\widehat\pi^g)$ is misspecified (but bounded away from $0$).

Combining the two parts gives
\begin{align}
\|\mathbb{E}[\phi^\pm_{t,z}(S;\widehat\eta^\pm)\mid \mathbf{X}]
-\overline\mu^\pm(t,z,\mathbf{X};\overline Q^\pm)\|_2=o_p(1).
\end{align}
Under Assumption~\ref{assump:second_stage} with $\delta_n=o_p(1)$, the final-stage regression therefore yields
$\|\widehat\mu^\pm(t,z,\cdot)-\overline\mu^\pm(t,z,\cdot;\overline Q^\pm)\|_2=o_p(1)$, and the sample-average estimator gives
$\widehat\psi^\pm(t,z)\to \overline\psi^\pm(t,z)$.
Thus the estimated (C)APO intervals converge to the induced (conservative) intervals and are asymptotically valid.
If $\overline Q^\pm=Q^\pm$, then $\overline\mu^\pm=\mu^\pm$ and the limits coincide with the sharp bounds.
\end{proof}


\subsection{Proof of Theorem \ref{thm:second_order_remainder_continuous}}

{
\renewcommand{\vspace}[1]{}
\renewcommand{\addvspace}[1]{}
\nuisanceReminderCont*
}

\begin{proof}
We mirror the proof of Theorem~\ref{thm:second_order_remainder} and highlight only the changes required for continuous $Z$.
Fix $(t,z)$ and suppress $(t,z)$ in the notation. As before, by cross-fitting, we may condition on the training folds and
treat $\widehat\eta$ as fixed when taking expectations over the held-out fold.

\paragraph{Key modification.}
For continuous $Z$, define
\begin{align}
A_h := \mathbf{1}_{[T=t]}K_h(Z-z),\qquad
\pi(Z,\mathbf{X}) := \pi^t(\mathbf{X})\,\pi^g(Z\mid \mathbf{X}),\qquad
\widehat\pi(Z,\mathbf{X}) := \widehat\pi^t(\mathbf{X})\,\widehat\pi^g(Z\mid \mathbf{X}),
\end{align}
so that the (true) localized selection weight is
\begin{align}
\kappa_{t,z,h}(S)=\frac{A_h}{\pi(Z,\mathbf{X})}.
\end{align}
The discrete-$Z$ algebra carries through with $A$ replaced by $A_h$ and $\pi(\mathbf{X})$ replaced by $\pi(Z,\mathbf{X})$.
The only substantive difference is that $L_2$-norms of kernel-weighted terms pick up a factor $h^{-1/2}$ via
$\int K_h(u)^2\,du = O(1/h)$.

\paragraph{A useful kernel moment bound.}
Under Assumptions~\ref{assump:overlap_bounded} and~\ref{assump:kernel}, there exists a constant $C<\infty$ such that, for any
square-integrable measurable function $G(S)$,
\begin{equation}
\label{eq:kernel_cs_bound}
\left\| \mathbb{E}\!\left[ \kappa_{t,z,h}(S)\,G(S)\mid \mathbf{X}\right]\right\|_2
\le \frac{C}{\sqrt{h}}\ \|G(S)\|_2.
\end{equation}
Indeed, by conditional Cauchy--Schwarz,
\begin{align}
\big(\mathbb{E}[\kappa_{t,z,h}G\mid \mathbf{X}]\big)^2
\le \mathbb{E}[\kappa_{t,z,h}^2\mid \mathbf{X}]\ \mathbb{E}[G^2\mid \mathbf{X}],
\end{align}
and $\mathbb{E}[\kappa_{t,z,h}^2\mid \mathbf{X}]$ is of order $1/h$ because $K_h^2$ integrates to $O(1/h)$ and
$\pi^t(\mathbf{X}),\pi^g(\cdot\mid \mathbf{X})$ are bounded away from $0$ (overlap).

\paragraph{Step 1: Conditional expectation decomposition.}
Write $\phi_h(S;\cdot)$ for Eq.~\eqref{eq:pseudo_outcome_unified} with $\omega_{z,h}(Z)=K_h(Z-z)$.
As in the discrete proof, take conditional expectations given $\mathbf{X}$ and use iterated expectations to replace the
in-sample hinge terms by their corresponding conditional-mean targets (evaluated at the cutoff used in the pseudo-outcome).
This yields a decomposition of the form
\begin{equation}
\label{eq:remainder_decomp_cont}
\mathbb{E}\!\left[\phi_h(S;\widehat\eta)-\phi_h(S;\eta)\mid \mathbf{X}\right]
=
\underbrace{\Big(\mu_{h,\widehat Q^+}(\mathbf{X})-\mu_h^+(\mathbf{X})\Big)}_{\text{cutoff-induced error}}
+
\underbrace{\mathbb{E}\!\left[\left(\frac{\pi(Z,\mathbf{X})}{\widehat\pi(Z,\mathbf{X})}-1\right)\kappa_{t,z,h}(S)\,\Delta_{\widehat Q^+}(S)\,\Bigm|\mathbf{X}\right]}_{\text{propensity $\times$ regression product term}},
\end{equation}
where $\mu_{h,\widehat Q^+}(\mathbf{X})$ denotes the bound functional induced by the cutoff $\widehat Q^+$ (holding the
remaining targets at their population values for that cutoff), and $\Delta_{\widehat Q^+}(S)$ collects the hinge-mean
regression discrepancies at cutoff $\widehat Q^+$ (the continuous-$Z$ analogue of the bracketed term in
Eq.~\eqref{eq:remainder_decomp} of the discrete proof).

\paragraph{Step 2: Bounding the product term.}
By overlap, $\widehat\pi(Z,\mathbf{X})$ is bounded away from $0$; hence
\begin{align}
\left\|\frac{\pi(Z,\mathbf{X})}{\widehat\pi(Z,\mathbf{X})}-1\right\|_2
\lesssim
\|\widehat\pi^t-\pi^t\|_2+\|\widehat\pi^g-\pi^g\|_2
= O_p(r_{n,\pi}),
\end{align}
where norms are taken over the arguments on which the nuisances are evaluated (here $(Z,\mathbf{X})$ for $\pi^g$).
Moreover, by definition of $r_{n,\gamma}$, $\|\Delta_{\widehat Q^+}(S)\|_2=O_p(r_{n,\gamma})$.
Applying Eq.~\eqref{eq:kernel_cs_bound} with $G(S):=\left(\frac{\pi}{\widehat\pi}-1\right)\Delta_{\widehat Q^+}(S)$ gives
\begin{align}
\left\|
\mathbb{E}\!\left[\left(\frac{\pi}{\widehat\pi}-1\right)\kappa_{t,z,h}\Delta_{\widehat Q^+}\mid \mathbf{X}\right]
\right\|_2
\le \frac{C}{\sqrt{h}}\left\|\left(\frac{\pi}{\widehat\pi}-1\right)\Delta_{\widehat Q^+}\right\|_2
= O_p\!\left(\frac{r_{n,\pi}\,r_{n,\gamma}}{\sqrt{h}}\right).
\end{align}

\paragraph{Step 3: Bounding the cutoff-induced term.}
The discrete proof bounds the cutoff-induced term using the envelope/FOC property of the cutoff and a second-order Taylor expansion, yielding a quadratic dependence on $\widehat Q^+-Q^+$. The same argument applies here pointwise in the arguments of the cutoff (the cutoff remains an optimizer of the same tail objective, now for the localized target), so
\begin{align}
|\mu_{h,\widehat Q^+}(\mathbf{X})-\mu_h^+(\mathbf{X})|
\lesssim
\mathbb{E}\!\left[\kappa_{t,z,h}(S)\,\big|\widehat Q^+(t,Z,\mathbf{X})-Q^+(t,Z,\mathbf{X})\big|^2\ \Bigm|\mathbf{X}\right].
\end{align}
Applying Eq.~\eqref{eq:kernel_cs_bound} with $G(S):=\big|\widehat Q^+(t,Z,\mathbf{X})-Q^+(t,Z,\mathbf{X})\big|^2$ and using the same bounded-moment simplification as in the discrete proof gives
\begin{align}
\|\mu_{h,\widehat Q^+}-\mu_h^+\|_2
=
O_p\!\left(\frac{r_{n,Q}^2}{\sqrt{h}}\right).
\end{align}

\paragraph{Conclusion.}
Combining Steps 2--3 in Eq.~\eqref{eq:remainder_decomp_cont} yields
\begin{align}
\left\|\mathbb{E}\!\left[\phi_h(S;\widehat\eta)-\phi_h(S;\eta)\mid \mathbf{X}\right]\right\|_2
=
O_p\!\left(\frac{r_{n,\pi}\,r_{n,\gamma}+r_{n,Q}^2}{\sqrt{h}}\right),
\end{align}
which is exactly the claim.
\end{proof}


\subsection{Proof of Corollary \ref{cor:quasi_oracle_and_inference_continuous}}

{
\renewcommand{\vspace}[1]{}
\renewcommand{\addvspace}[1]{}
\quasiOracleRatesCont*
}

\begin{proof}
We follow the proof of Corollary~\ref{cor:quasi_oracle_and_inference}, replacing Theorem~\ref{thm:second_order_remainder}
by Theorem~\ref{thm:second_order_remainder_continuous} and tracking the kernel-induced scaling.

\paragraph{CAPO rate.}
Let $m^+_{t,z,h}(\mathbf{x}):=\mathbb{E}[\widehat\phi^+_{t,z,h}\mid \mathbf{X}=\mathbf{x}]$ denote the conditional mean of
the (cross-fitted) localized pseudo-outcome. By Assumption~\ref{assump:second_stage},
\begin{align}
\|\widehat\mu_h^+(t,z,\cdot)-m^+_{t,z,h}(\cdot)\|_2 = O_p(\delta_n).
\end{align}
By the triangle inequality,
\begin{align}
\|\widehat\mu_h^+(t,z,\cdot)-\mu_h^+(t,z,\cdot)\|_2
\le
\|\widehat\mu_h^+(t,z,\cdot)-m^+_{t,z,h}(\cdot)\|_2
+\|m^+_{t,z,h}(\cdot)-\mu_h^+(t,z,\cdot)\|_2.
\end{align}
The second term is exactly the conditional nuisance-induced bias controlled by
Theorem~\ref{thm:second_order_remainder_continuous}, giving the stated CAPO rate.

\paragraph{APO rate and $\sqrt{nh}$ asymptotic normality.}
Recall $\widehat\psi_h^+(t,z)=\mathbb{E}_n[\widehat\phi^+_{t,z,h}]$. Decompose
\begin{align}
\widehat\psi_h^+(t,z)-\psi_h^+(t,z)
=
(\mathbb{E}_n-\mathbb{E})[\phi^+_{t,z,h}(S;\eta)]
+\mathbb{E}\!\left[\phi^+_{t,z,h}(S;\widehat\eta)-\phi^+_{t,z,h}(S;\eta)\right]
+R_{n,h},
\end{align}
where $R_{n,h}:=(\mathbb{E}_n-\mathbb{E})[\phi^+_{t,z,h}(S;\widehat\eta)-\phi^+_{t,z,h}(S;\eta)]$.

Under Assumption~\ref{assump:kernel} and overlap, $\mathrm{Var}(\phi^+_{t,z,h}(S;\eta))=O(1/h)$, so
$(\mathbb{E}_n-\mathbb{E})[\phi^+_{t,z,h}(S;\eta)]=O_p((nh)^{-1/2})$ by the CLT.
With cross-fitting and the same conditioning argument as in the discrete proof, $R_{n,h}=o_p((nh)^{-1/2})$.

The bias term is controlled by Theorem~\ref{thm:second_order_remainder_continuous} after integrating over $\mathbf{X}$,
yielding the stated APO rate. If additionally $r_{n,\pi}r_{n,\gamma}+r_{n,Q}^2=o_p(n^{-1/2})$, then the bias term and
$R_{n,h}$ are $o_p((nh)^{-1/2})$, implying
\begin{align}
\sqrt{nh}\big(\widehat\psi_h^+(t,z)-\psi_h^+(t,z)\big)
=
\sqrt{nh}(\mathbb{E}_n-\mathbb{E})[\phi^+_{t,z,h}(S;\eta)]
+o_p(1)
\ \rightsquigarrow\
\mathcal{N}(0,V_h^+(t,z)),
\end{align}
with $V_h^+(t,z)=\mathrm{Var}(\sqrt{h}\,\phi^+_{t,z,h}(S;\eta))$.
The final undersmoothing statement follows by adding/subtracting $\psi^+(t,z)$ and using the assumed bias condition.
\end{proof}


\subsection{Proof of Proposition \ref{prop:sharpness_continuous}}

{
\renewcommand{\vspace}[1]{}
\renewcommand{\addvspace}[1]{}
\sharpnessGuaranteeCont*
}

\begin{proof}
The argument is identical to the proof of Proposition~\ref{prop:sharpness}, replacing
Corollary~\ref{cor:quasi_oracle_and_inference} by Corollary~\ref{cor:quasi_oracle_and_inference_continuous}.
Under the stated assumptions, $\delta_n=o_p(1)$ and
$\frac{r_{n,\pi}r_{n,\gamma}+r_{n,Q}^2}{\sqrt{h}}=o_p(1)$, hence both CAPO endpoints converge in $L_2$ to the sharp
kernel-localized endpoints $\mu_h^\pm(t,z,\cdot)$. The APO convergence follows from the APO rate in
Corollary~\ref{cor:quasi_oracle_and_inference_continuous}. Finally, if the smoothing bias vanishes at the stated rate,
the same conclusion holds for the pointwise (unsmoothed) bounds.
\end{proof}


\subsection{Proof of Corollary \ref{cor:validity_misspecified_cutoffs_continuous}}
{
\renewcommand{\vspace}[1]{}
\renewcommand{\addvspace}[1]{}
\validityCutoffsCont*
}
\begin{proof}
We adapt the proof of Corollary~\ref{cor:validity} and indicate only the continuous-$Z$ differences.

\paragraph{Step 1: Any cutoffs induce a valid (conservative) localized interval.}
Fix $(t,z,\mathbf{x})$. For each exposure level $u$, the discrete-$Z$ proof shows (via the same Rockafellar-Uryasev tail
objectives) that evaluating the tail objective at an arbitrary cutoff $\overline Q^\pm(t,u,\mathbf{x})$ yields conservative
endpoints $\overline\mu^\pm(t,u,\mathbf{x};\overline Q^\pm)$ that contain the sharp pointwise endpoints $\mu^\pm(t,u,\mathbf{x})$.

Kernel localization preserves this ordering because $K_h(\cdot)\ge 0$ and integrates to $1$.
Indeed, the continuous-$Z$ localized targets are obtained by the same conditional-expectation construction as in
Eq.~\eqref{eq:mu_psi_h_defs}, and the weight $\kappa_{t,z,h}(S)$ transports pointwise statements in $u$ into their localized
analogues around $z$. Therefore,
\begin{align}
\overline\mu_h^-(t,z,\mathbf{x};\overline Q^-)\ \le\ \mu_h^-(t,z,\mathbf{x})
\ \le\ \mu_h^+(t,z,\mathbf{x})\ \le\ \overline\mu_h^+(t,z,\mathbf{x};\overline Q^+),
\end{align}
such that the CAPO interval is valid (though not necessarily sharp). The APO claim follows by taking expectations over $\mathbf{X}$.

\paragraph{Step 2: Convergence to the induced (conservative) localized bounds.}
The convergence argument follows Step~2 of the discrete-$Z$ proof, with the single change that kernel-weighted terms
are controlled using the bound in Eq.~\eqref{eq:kernel_cs_bound} (hence the extra factor $h^{-1/2}$ in intermediate inequalities).
Under $\widehat Q^\pm\to \overline Q^\pm$ in $L_2$ and either
(i) $(\widehat\pi^t,\widehat\pi^g)$ consistent or (ii) $(\widehat\gamma_u^\pm,\widehat\gamma_l^\pm)$ consistent for the
targets induced by $\overline Q^\pm$, the same decomposition yields
\begin{align}
\left\|\mathbb{E}\!\left[\phi^\pm_{t,z,h}(S;\widehat\eta^\pm)\mid \mathbf{X}\right]
-\overline\mu_h^\pm(t,z,\mathbf{X};\overline Q^\pm)\right\|_2=o_p(1).
\end{align}
With Assumption~\ref{assump:second_stage} and $\delta_n=o_p(1)$, the second-stage regression therefore implies
$\|\widehat\mu_h^\pm(t,z,\cdot)-\overline\mu_h^\pm(t,z,\cdot;\overline Q^\pm)\|_2=o_p(1)$, and the sample-average estimator yields $\widehat\psi_h^\pm(t,z)\to \overline\psi_h^\pm(t,z;\overline Q^\pm)$. Thus the estimated intervals converge to the induced (conservative) localized intervals and remain asymptotically valid. If $\overline Q^\pm=Q^\pm$, these limits coincide
with the sharp localized bounds.
\end{proof}

\newpage
\section{Practical considerations}
\label{sec:appendix_use_case}

\subsection{Applications of exposure mappings}

Our framework accommodates a range of \emph{exposure mappings} $g(\cdot)$ that reduce the (typically high-dimensional) vector of neighbors' treatments into a low-dimensional exposure variable for unit $i$. The choice of $g$ should be guided by substantive knowledge about how interference operates, and by what is measured in the data. Below, we summarize common mappings and settings where they arise.

\circled{1} \textbf{Applications of the weighted-mean exposure}

A common exposure mapping is the number of treated neighbors $g(T_{\mathcal N_i}) \;=\; \sum_{j\in\mathcal N_i} T_j$. This is natural when spillovers scale approximately with the number of treated contacts: repeated encouragement from multiple peers can increase salience, multiple treated farmers can raise local demonstration intensity, and multiple trained coworkers can increase adoption of a workflow tool. In spatial policy applications, a count of nearby treated sites (e.g., treated intersections or corridors) can proxy local exposure intensity to infrastructure changes. 

A more realistic mapping often weights neighbors by interaction intensity: $g(T_{\mathcal N_i}) \;=\; \sum_{j\in\mathcal N_i} w_{ij} T_j$,
$w_{ij}\ge 0,$ where $w_{ij}$ encodes geographic distance decay, communication frequency, tie strength, or mobility flows. This is common in epidemiology (contact rates), spatial economics (commuting flows), and online platforms (interaction networks). In transport and infrastructure settings, weights based on travel time or distance can encode that nearer interventions plausibly matter more.

When neighborhood size varies, a normalized mapping captures saturation: $g(T_{\mathcal N_i}) \;=\; \frac{1}{|\mathcal N_i|}\sum_{j\in\mathcal N_i} T_j.$ This is widely used in education (fraction of classmates treated), workplaces (share of coworkers trained), and community programs (share of households reached by a campaign). Proportion-based mappings are also natural in information environments (online or offline) where the fraction of one's social neighborhood exposed to an intervention (e.g., a correction or informational nudge) shapes beliefs or behavior.

\circled{2} \textbf{Applications of the threshold exposure}

A simple mapping is $g(T_{\mathcal N_i}) \;=\; \mathbf{1}\Big\{\sum_{j\in\mathcal N_i} T_j \ge 1\Big\},$ i.e., whether unit $i$ has at least one treated neighbor. This is appropriate when spillovers plausibly operate through \emph{presence} rather than intensity: diffusion of a new practice or technology can start once a single close contact adopts it (e.g., demonstration effects or peer-to-peer referrals). In public health, the presence of a vaccinated (or otherwise treated) contact may affect risk via reduced transmission in small networks (e.g., household or close-contact structures), where the key margin is whether at least one relevant contact is treated.

A thresholded saturation mapping, $g(T_{\mathcal N_i}) \;=\; \mathbf{1}\Big\{\frac{1}{|\mathcal N_i|}\sum_{j\in\mathcal N_i} T_j \ge c\Big\},$ is appropriate when spillovers exhibit \emph{nonlinearities} such as coordination, social norms, or capacity constraints. Examples include collective action (private incentives change after a critical mass), technology standards (compatibility benefits after adoption passes a threshold), and community compliance contexts (program effects emerge only once participation exceeds a minimum level). In behavioral climate interventions, a threshold can represent a social-norm mechanism: behavior changes once individuals perceive that ``most peers'' act.

\circled{3} \textbf{High-order spillover effects}

If interference operates through longer paths than captured by $\mathcal N_i$ (e.g., two-hop network effects, market-level general equilibrium, or broader media spillovers), then a strictly local exposure definition may be inadequate. Enlarging $\mathcal N_i$ or using weighted kernels can address this conceptually, but typically worsens overlap and can further widen bounds.

\subsection{Limitations of our partial-identification bounds}

Our partial-identification results yield identification-robust statements under interference and limited structure. That said, the bounds come with concrete limitations.

First of all, our bounds can be wide when the data are weakly informative.
Partial identification is conservative by construction. When counterfactual information is scarce, e.g., rare high-saturation neighborhoods, limited overlap across exposure regimes, or strongly clustered assignment, the bounds may be wide. This reflects genuine lack of information under the maintained assumptions, not an estimation failure. While our approach targets robustness, finite samples can be sensitive to rare exposure levels and to the quality of nuisance estimation (e.g., outcome regression and any assignment/exposure models). Heavy-tailed outcomes or highly imbalanced exposures can amplify instability, motivating careful overlap diagnostics, effective-sample-size reporting by exposure regime, and sensitivity checks across learners.
Furthermore, network data is often incomplete (missing ties, mismeasured distances, unknown interaction strengths). Although we assume a fully observed network in our paper, errors in $T_j$, $\mathcal N_i$, or weights $w_{ij}$ propagate directly into $g(T_{\mathcal N_i})$. Because the bounds are defined with respect to observed exposures, measurement error can attenuate or distort the effective exposure regimes and complicate substantive interpretation.

\newpage
\section{Implementation details}
\label{sec:appendix_implementation}

\subsection{Data generation}

We study the finite-sample performance of our proposed bound estimators in a network interference setting with potentially misspecified exposure mappings. The data-generating process is designed to closely mirror the structural assumptions of Section~\ref{sec:method} while allowing for controlled violations of the exposure mapping.

\textbf{Units, covariates, and network.}
We observe $N$ units indexed by $i=1,\dots,n$. Each unit is endowed with a $d$-dimensional covariate vector
$\mathbf{X}_i \sim \mathcal{U}(-1, 1)$,
drawn independently across units. Units are connected through a known undirected network $G=(V,E)$, where neighborhoods are defined as $\mathcal{N}_i = \{j : (i,j)\in E\}$ and node-specific degrees are denoted by $n_i = |\mathcal{N}_i|$. 

Overall, we conduct six experiments. We generate each three networks with $N=1000$ nodes and a $1$-dimensional covariate (small datasets) and three networks with $N=6000$ nodes and $6$-dimensional covariates (large datasets). Each network is chosen to demonstrate the theoretical properties of our framework in dealing with different forms of exposure misspecification: we employ a random Erdős–Rényi network for testing exposure mapping \circled{1}, a scale-free Barabási–Albert network for exposure mapping \circled{2}, and community-structured stochastic block model exposure mappings \circled{3}. We elaborate on the design in Subsection~\ref{sec:appendix_implementation_network}.

\textbf{Treatment assignment.}
Each unit receives a binary treatment $T_i\in\{0,1\}$. Treatments are assigned independently conditional on covariates according to a logistic propensity score model $\pi^t(\mathbf{x}) := \mathbb{P}(T_i=1\mid \mathbf{X}_i=\mathbf{x}) = \operatorname{logit}^{-1}(\beta_T^\top \mathbf{x})$, where $\beta_T\in\mathbb{R}^d$ is fixed across simulations.

\textbf{Potential outcomes and observed data.}
Potential outcomes depend on individual treatment, exposure, and covariates according to
\begin{align}
Y_i(t,z) = m(t,z,\mathbf{X}_i) + \varepsilon_i,
\qquad
\varepsilon_i \sim \mathcal{N}(0,\sigma^2),
\end{align}
with
\begin{align}
m(t,z,\mathbf{x}) = \tau\, t + \delta\, z + \gamma\, t z + f(\mathbf{x}),
\end{align}
where $\tau$ captures the direct effect of treatment, $\delta$ the spillover effect of exposure,
$\eta$ allows for treatment--exposure interaction, and $f(\mathbf{x})$ is a nonlinear baseline
function of covariates. The observed outcome is given by $Y_i = Y_i(T_i, Z_i^\ast)$, so that interference operates through the true exposure mapping $g^\ast$.

\textbf{Target estimands.}
Our primary targets are the conditional average potential outcomes (CAPOs)
\begin{align}
\mu(t,z,\mathbf{x}) := \mathbb{E}[Y(t,z)\mid \mathbf{X}=\mathbf{x}],
\end{align}
and their induced causal effects. In particular, we consider
(i)~conditional direct effects $\mu(1,z,\mathbf{x})-\mu(0,z,\mathbf{x})$,
(ii)~conditional spillover effects $\mu(t,z_1,\mathbf{x})-\mu(t,z_0,\mathbf{x})$, as well as their averaged (APO) counterparts. 

\textbf{Experimental variations.}
Across simulation scenarios, we vary the network density, the degree of exposure-mapping misspecification through $\varepsilon$, the discreteness or continuity of $Z$, and the sample size $n$. All reported results are averaged over multiple Monte Carlo repetitions.

\begin{table}[h]
\centering
\caption{Simulation design and parameter specifications.}
\label{tab:simulation_specs}
\begin{tabular}{ll}
\toprule
\textbf{Component} & \textbf{Specification / Values} \\
\midrule
Units (nodes) & $N = 3000 / 6000$ \\
Covariate dimension & $d = 1 / 6$\\
Treatment propensity & $\beta_T = 0.8$ \\
Direct effect & $\tau = 0.8$\\
Spillover effect & $\delta = 0.6$\\
Interaction & $\gamma = 0.2$\\
Outcome model nonlinearity & $0.6 \tanh(X) + 0.4\sin(X) - 0.2X^2$ \\
Noise & $\xi_i\sim\mathcal{N}(0,1)$ \\
Kernel bandwidth & $h = 0.1$ \\
Mean misspecification & $ \varepsilon = 0.03$\\
Threshold misspecifition & $c^{\ast}=0.45$\\
Cross-fitting folds & $K=5$ \\
Runs & $20$ \\
\bottomrule
\end{tabular}
\end{table}

\subsection{Choice of network structure}
\label{sec:appendix_implementation_network}

We deliberately vary the underlying network structure across simulation scenarios to ensure that each form of exposure mapping misspecification is evaluated in a setting where it is substantively meaningful. Different misspecifications interact with distinct structural properties of networks, and using a single network model throughout would mask or attenuate these effects. 

Specifically, we employ the following networks: \circled{1}~\textit{Weighted versus unweighted mean exposure:} Erdős–Rényi (ER) networks, whose homogeneous degree distribution isolates the effect of heterogeneous influence weights from confounding degree heterogeneity; \circled{2}~\textit{Threshold-based exposure mappings:} scale-free networks generated by a Barabási–Albert (BA) process, where heavy-tailed degree distributions imply that small shifts in the threshold can lead to substantial misclassification of exposure, particularly for highly connected units; \circled{3}~\textit{Higher order spillovers: }stochastic block models (SBMs) with pronounced community structure, where spillovers naturally propagate beyond direct neighbors through short multi-step paths, making truncation of the exposure radius consequential. 

Across all scenarios, the network choice is therefore tailored to highlight the specific failure mode induced by the corresponding exposure mapping misspecification, rather than to optimize estimator performance.

\subsection{Implementation}

We implement our experiments in Python. All code for replication is available in our GitHub repository at \url{https://github.com/m-schroder/ExposureMisspecification}.

We estimate nuisance components via cross‑fitted learners: the treatment propensity is fit with gradient‑boosted classifiers; the exposure density uses a Gaussian‑mixture conditional model for continuous exposures (or multinomial/gradient boosting for discrete cases). We fit our quantile regression and the conditional expectations $\gamma$ with XGBoost models. For the second‑stage regression, we as well employ XGBoost with cross‑fitting. We select hyperparameters through cross‑validation (log‑loss for propensity/exposure models, pinball loss for quantile models, and MSE for regression).

\end{document}